\documentclass[letterpaper,english]{article}
\usepackage[T1]{fontenc}
\usepackage[latin9]{inputenc}
\usepackage{color}
\usepackage{array}
\usepackage{float}
\usepackage{booktabs}
\usepackage{url}
\usepackage{multirow}
\usepackage{amsmath}
\usepackage{amsthm}
\usepackage{amssymb}
\usepackage{graphicx}
\usepackage[authoryear]{natbib}
\PassOptionsToPackage{normalem}{ulem}
\usepackage{ulem}

\makeatletter

\pdfpageheight\paperheight
\pdfpagewidth\paperwidth

\providecommand{\tabularnewline}{\\}

\theoremstyle{plain}
\newtheorem{thm}{\protect\theoremname}

\usepackage[final]{nips_2018}
\usepackage[T1]{fontenc}    % use 8-bit T1 fonts

\usepackage{algorithmic}
\usepackage{algorithm}
\usepackage{amsmath}
\usepackage{amsthm}
\usepackage{grffile}
\usepackage{hyperref}       % hyperlinks
\usepackage{xcolor}
\usepackage{lipsum}
\newcommand\myshade{55}
\colorlet{mylinkcolor}{blue}
\colorlet{mycitecolor}{blue}
\colorlet{myurlcolor}{blue!50}

\hypersetup{
  linkcolor  = mylinkcolor!\myshade!black,
  citecolor  = mycitecolor!\myshade!black,
  urlcolor   = myurlcolor!\myshade!black,
  colorlinks = true,
}

\usepackage{url}            % simple URL typesetting
\usepackage{booktabs}       % professional-quality tables
\usepackage{amsfonts}       % blackboard math symbols\

\usepackage{enumitem}
\setlist{topsep=0.5mm, parsep=0.7mm, partopsep=0.5mm}
\setlist[enumerate]{leftmargin=2em}

\usepackage{nicefrac}       % compact symbols for 1/2, etc.
\usepackage{microtype}      % microtypography

\usepackage{natbib}
\usepackage{soul}
\usepackage{subfig}

\usepackage{tikz}
\usetikzlibrary{overlay-beamer-styles,arrows.meta}
\usetikzlibrary{arrows,calc,shapes.geometric,decorations.markings}

\usepackage{thm-restate}
\usepackage{thmtools}
\usepackage{wrapfig}
\usepackage{url}

\let\oldenumerate=\enumerate
\def\enumerate{
\oldenumerate
\setlength{\itemsep}{0pt}
}
\let\olditemize=\itemize
\def\itemize{
\olditemize
\setlength{\itemsep}{0pt}
}

\allowdisplaybreaks

\@ifundefined{showcaptionsetup}{}{%
 \PassOptionsToPackage{caption=false}{subfig}}
\usepackage{subfig}
\makeatother

\usepackage{babel}
\providecommand{\theoremname}{Theorem}

\begin{document}
\global\long\def\bg{\boldsymbol{g}}

\global\long\def\bxi{\boldsymbol{\xi}}

\global\long\def\btau{\boldsymbol{\tau}}

\global\long\def\bmu{\boldsymbol{\mu}}
\global\long\def\fssdhp{\widehat{F_{p}^{2}}}
\global\long\def\fssdhq{\widehat{F_{q}^{2}}}

\global\long\def\fssdp{F_{p}^{2}}
\global\long\def\fssdq{F_{q}^{2}}

\global\long\def\umehpr{\widehat{U_{P}^{2}}}
\global\long\def\umehqr{\widehat{U_{Q}^{2}}}

\global\long\def\umepr{U_{P}^{2}}
\global\long\def\umeqr{U_{Q}^{2}}

\global\long\def\nmehpr{\widehat{\mathrm{NME}_{P}^{2}}}
\global\long\def\nmehqr{\widehat{\mathrm{NME}_{Q}^{2}}}

\global\long\def\mehp{\widehat{\mathrm{ME}_{p}^{2}}}
\global\long\def\mehq{\widehat{\mathrm{ME}_{q}^{2}}}

\global\long\def\psipv{\psi_{V}^{P}}
\global\long\def\psiqw{\psi_{W}^{Q}}

\global\long\def\psirv{\psi_{V}^{R}}
\global\long\def\psirw{\psi_{W}^{R}}

\global\long\def\psihpv{\hat{\psi}_{V}^{P}}
\global\long\def\psihqw{\hat{\psi}_{W}^{Q}}

\global\long\def\psihrv{\hat{\psi}_{V}^{R}}
\global\long\def\psihrw{\hat{\psi}_{W}^{R}}

\global\long\def\covpv{C_{V}^{P}}
\global\long\def\covqw{C_{W}^{Q}}

\global\long\def\covrv{C_{V}^{R}}
\global\long\def\covrw{C_{W}^{R}}

\global\long\def\covhpv{\hat{C}_{V}^{P}}
\global\long\def\covhqw{\hat{C}_{W}^{Q}}

\global\long\def\covhrv{\hat{C}_{V}^{R}}
\global\long\def\covhrw{\hat{C}_{W}^{R}}

\global\long\def\relume{\mathrm{Rel\textnormal{-}UME}}
\global\long\def\relfssd{\mathrm{Rel\textnormal{-}FSSD}}
\global\long\def\relmmd{\mathrm{Rel\textnormal{-}MMD}}

%------------------------------------------------
\newcommand{\ourtitle}{Informative Features for Model Comparison}

\theoremstyle{remark} 
\newtheorem{remark}{Remark} 
\date{}

\title{\ourtitle}
% The \author macro works with any number of authors. There are two
% commands used to separate the names and addresses of multiple
% authors: \And and \AND.
%
% Using \And between authors leaves it to LaTeX to determine where to
% break the lines. Using \AND forces a line break at that point. So,
% if LaTeX puts 3 of 4 authors names on the first line, and the last
% on the second line, try using \AND instead of \And before the third
% author name.
\author{
  Wittawat Jitkrittum\\
  Max Planck Institute for Intelligent Systems\\
  %Cranberry-Lemon University\\
  %Pittsburgh, PA 15213 \\
  \small{\url{wittawat@tuebingen.mpg.de}} \\
  %% examples of more authors
%
 \And
Heishiro Kanagawa\\
Gatsby Unit, UCL\\
\small{\url{heishirok@gatsby.ucl.ac.uk}} \\
 \And
Patsorn Sangkloy\\
 Georgia Institute of Technology\\
\small{\url{patsorn_sangkloy@gatech.edu}} \\
\And
James Hays\\
Georgia Institute of Technology\\
\small{\url{hays@gatech.edu}} \\
\And
Bernhard Sch\"{o}lkopf\\
Max Planck Institute for Intelligent Systems\\
\small{\url{bernhard.schoelkopf@tuebingen.mpg.de}} \\
\And
Arthur Gretton\thanks{Arthur Gretton's ORCID ID: 0000-0003-3169-7624.}  \\
Gatsby Unit, UCL \\
\small{\url{arthur.gretton@gmail.com}}
}

\maketitle
\begin{abstract}
Given two candidate models, and a set of target observations, we address
the problem of measuring the relative goodness of fit of the two models.
We propose two new statistical tests which are nonparametric, computationally
efficient (runtime complexity is linear in the sample size), and interpretable.
As a unique advantage, our tests can produce a set of examples (informative
features) indicating the regions in the data domain where one model
fits significantly better than the other. In a real-world problem
of comparing GAN models, the test power of our new test matches that
of the state-of-the-art test of relative goodness of fit, while being
one order of magnitude faster.\vspace{-3mm}
\end{abstract}

\section{Introduction}

\vspace{-3mm}One of the most fruitful areas in recent machine learning
research has been the development of effective generative models for
very complex and high dimensional data. Chief among these have been
the generative adversarial networks \citep{GooPouMirXuWar2014,wgan,NowBotRyo16},
where samples may be generated without an explicit generative model
or likelihood function. A related thread has emerged in the statistics
community with the advent of Approximate Bayesian Computation, where
simulation-based models without closed-form likelihoods are widely
applied in bioinformatics applications \citep[see][for a review]{LinGurtDutKasCor17}.
In these cases, we might have several competing models, and wish
to evaluate which is the better fit for the data.

The problem of model criticism is traditionally defined as follows:
how well does a model $Q$ fit a given sample $Z_{n}:=\{\boldsymbol{z}_{i}\}_{i=1}^{n}\stackrel{i.i.d.}{\sim}R$?
This task can be addressed in two ways: by comparing samples $Y_{n}:=\{\boldsymbol{y}_{i}\}_{i=1}^{n}$
from the model $Q$ and data samples, or by directly evaluating the
goodness of fit of the model itself. In both of these cases, the tests
have a null hypothesis (that the model agrees with the data), which
they will reject given sufficient evidence. Two-sample tests fall
into the first category: there are numerous nonparametric tests which
may be used \citep{AlbaFernandez2008,GreBorRasSchSmo2012,FriRaf79,SzeRiz04,Rosenbaum05,HarBacMou08,HalTaj02,JitSzaChwGre2016},
and recent work in applying two-sample tests to the problem of model
criticism \citep{lloyd2015statistical}. A second approach requires
the model density $q$ explicitly. In the case of simple models for
which normalisation is not an issue (e.g., checking for Gaussianity),
several tests exist \citep{BaringhausHenze88,SzeRiz2005}; when a
model density is known only up to a normalisation constant, tests
of goodness of fit have been developed using a Stein-based divergence
\citep{ChwStrGre2016,LiuLeeJor2016,JitXuSzaFukGre2017}.

An issue with the above notion of model criticism, particularly in
the case of modern generative models, is that \emph{any} hypothetical
model $Q$ that we design is likely a poor fit to the data. Indeed,
as noted in \citet[Section 5.5]{YamWuTsaTakSal2018}, comparing samples
from various Generative Adversarial Network (GAN) models \citep{GooPouMirXuWar2014}
to the reference sample $Z_{n}$ by a variant of the Maximum Mean
Discrepancy (MMD) test \citep{GreBorRasSchSmo2012} leads to the trivial
conclusion that all models are wrong \citep{Box76}, i.e., $H_{0}\colon Q=R$
is rejected by the test in all cases. A more relevant question in
practice is thus: ``Given two models $P$ and $Q$, which is closer
to $R$, and in what ways?'' This is the problem we tackle in this
work.

To our knowledge, the only nonparametric statistical test of \emph{relative}
goodness of fit is the $\relmmd$ test of \citet{BouBelBlaAntGre2015},
based on the maximum mean discrepancy \citep[MMD,][]{GreBorRasSchSmo2012}.
While shown to be practical (e.g., for comparing network architectures
of generative networks), two issues remain to be addressed. Firstly,
its runtime complexity is quadratic in the sample size $n$, meaning
that it can be applied only to problems of moderate size. Secondly
and more importantly, it does not give an indication of where one
model is better than the other. This is essential for model comparison:
in practical settings, it is highly unlikely that one model will be
uniformly better than another in all respects: for instance, in hand-written
digit generation, one model might produce better ``3''s, and the
other better ``6''s. The ability to produce a few examples which
indicate regions (in the data domain) in which one model fits better
than the other will be a valuable tool for model comparison. This
type of interpretability is useful especially in learning generative
models with GANs, where the ``mode collapse'' problem is widespread
\citep{SalGooZarCheRad2016,SriValRusGutSut2017}. The idea of generating
such distinguishing examples (so called \emph{test locations}) was
explored in \citet{JitSzaChwGre2016,JitXuSzaFukGre2017} in the context
of model criticism and two-sample testing.

In this work, we propose two new linear-time tests for relative goodness-of-fit.
In the first test, the two models $P,Q$ are represented by their
two respective samples $X_{n}$ and $Y_{n}$, and the test generalises
that of \citet{JitSzaChwGre2016}. In the second, the test has access
to the probability density functions $p,q$ of the two respective
candidate models $P,Q$ (which need only be known up to normalisation),
and is a three-way analogue of the test of \citet{JitXuSzaFukGre2017}.
In both cases, the tests return locations indicating where one model
outperforms the other. We emphasise that the practitioner must choose
the model ordering, since as noted earlier, this will determine the
locations that the test prioritises. We further note that the two
tests complement each other, as both address different aspects of
the model comparison problem. The first test simply finds the location
where the better model produces mass closest to the test sample: a
worse model can produce too much mass, or too little. The second test
does not address the overall probability mass, but rather the shape
of the model density: specifically, it penalises the model whose derivative
log density differs most from the target (the interpretation is illustrated
in our experiments). In the experiment on comparing two GAN models,
we find that the performance of our new test matches that of $\relmmd$
while being one order of magnitude faster. Further, unlike the popular
Fr\'{e}chet Inception Distance (FID) \citep{HeuRamUntNesHoc2017}
which can give a wrong conclusion when two GANs have equal goodness
of fit, our proposed method has a well-calibrated threshold, allowing
the user to flexibly control the false positive rate.\vspace{-3mm}

\section{Measures of Goodness of Fit}

\vspace{-2mm}In the proposed tests, we test the relative goodness
of fit by comparing the relative magnitudes of two distances, following
\citet{BouBelBlaAntGre2015}. More specifically, let $D(P,R)$ be
a discrepancy measure between $P$ and $R$. Then, the problem can
be formulated as a hypothesis test proposing $H_{0}\colon D(P,R)\le D(Q,R)$
against $H_{1}\colon D(P,R)>D(Q,R)$. This is the approach taken by
\citeauthor{BouBelBlaAntGre2015} who use the MMD as $D$, resulting
in the relative MMD test ($\relmmd$). The proposed $\relume$ and
$\relfssd$ tests are based on two recently proposed discrepancy measures
for $D$: the Unnormalized Mean Embeddings (UME) statistic \citep{ChwRamSejGre2015,JitSzaChwGre2016},
and the Finite-Set Stein Discrepancy (FSSD) \citep{JitXuSzaFukGre2017},
for the sample-based and density-based settings, respectively. We
first review UME and FSSD. We will extend these two measures to construct
two new relative goodness-of-fit tests in Section \ref{sec:new_mctests}.
We assume throughout that the probability measures $P,Q,R$ have a
common support $\mathcal{X}\subseteq\mathbb{R}^{d}$. 

\textbf{The Unnormalized Mean Embeddings (UME) Statistic} \label{sec:ume}UME
is a (random) distance between two probability distributions \citep{ChwRamSejGre2015}
originally proposed for two-sample testing for $H_{0}\colon Q=R$
and $H_{1}\colon Q\neq R$. Let $k_{Y}\colon\mathcal{X}\times\mathcal{X}\to\mathbb{R}$
be a positive definite kernel. Let $\mu_{Q}$ be the mean embedding
of $Q$, and is defined such that $\mu_{Q}(\boldsymbol{w}):=\mathbb{E}_{\boldsymbol{y}\sim Q}k(\boldsymbol{y},\boldsymbol{w})$
 (assumed to exist) \citep{SmoGreSonSch2007}. \citet{GreBorRasSchSmo2012}
shows that when $k_{Y}$ is characteristic \citep{SriFukLan2011},
the Maximum Mean Discrepancy (MMD) \emph{witness function} $\mathrm{wit}_{Q,R}(\boldsymbol{w}):=\mu_{Q}(\boldsymbol{w})-\mu_{R}(\boldsymbol{w})$
is a zero function if and only if $Q=R$. Based on this fact, the
UME statistic evaluates the squared witness function at $J_{q}$ test
locations $W:=\{\boldsymbol{w}_{j}\}_{j=1}^{J_{q}}\subset\mathcal{X}$
to determine whether it is zero. Formally, the population squared
UME statistic is defined as $U^{2}(Q,R):=\frac{1}{J}\sum_{j=1}^{J}(\mu_{Q}(\boldsymbol{w}_{j})-\mu_{R}(\boldsymbol{w}_{j}))^{2}$.
For our purpose, it will be useful to rewrite the UME statistic as
follows. Define the feature function $\psi_{W}(\boldsymbol{y}):=\frac{1}{\sqrt{J_{q}}}\left(k_{Y}(\boldsymbol{y},\boldsymbol{w}_{1}),\ldots,k_{Y}(\boldsymbol{y},\boldsymbol{w}_{J_{q}})\right)^{\top}\in\mathbb{R}^{J_{q}}$.
Let $\psiqw:=\mathbb{E}_{\boldsymbol{y}\sim Q}[\psi_{W}(\boldsymbol{y})]$,
and its empirical estimate $\psihqw:=\frac{1}{n}\sum_{i=1}^{n}\psi_{W}(\boldsymbol{y}_{i})$.
The squared population UME statistic is equivalent to $U^{2}(Q,R):=\|\psiqw-\psirw\|_{2}^{2}.$
For $W\sim\eta$ where $\eta$ is a distribution with a density, Theorem
2 in \citet{ChwRamSejGre2015} states that if $k_{Y}$ is real analytic,
integrable, and characteristic, then $\eta$-almost surely $\|\psiqw-\psirw\|_{2}^{2}=0$
if and only if $Q=R$. In words, under the stated conditions, $U(Q,R):=U_{Q}$
defines a distance between $Q$ and $R$ (almost surely).\footnote{In this work, since the distance is always measured relative to the
data generating distribution $R$, we write $U_{Q}$ instead of $U(Q,R)$
to avoid cluttering the notation.} A consistent unbiased estimator is $\umehqr=\frac{1}{n(n-1)}\left[\|\sum_{i=1}^{n}[\psi_{W}(\boldsymbol{y}_{i})-\psi_{W}(\boldsymbol{z}_{i})]\|^{2}-\sum_{i=1}^{n}\|\psi_{W}(\boldsymbol{y}_{i})-\psi_{W}(\boldsymbol{z}_{i})\|^{2}\right]$,
which clearly can be computed in $\mathcal{O}(n)$ time. \citet{JitSzaChwGre2016}
proposed optimizing the test locations $W$ and $k_{Y}$ so as to
maximize the test power (i.e., the probability of rejecting $H_{0}$
when it is false) of the two-sample test with the normalized version
of the UME statistic. It was shown that the optimized locations give
an interpretable indication of where $Q$ and $R$ differ in the input
domain $\mathcal{X}$.

\textbf{The Finite-Set Stein Discrepancy (FSSD)} \label{sec:fssd}
FSSD is a discrepancy between two density functions $q$ and $r$.
Let $\mathcal{X}\subseteq\mathbb{R}^{d}$ be a connected open set.
Assume that $Q,R$ have probability density functions denoted by $q,r$
respectively. Given a positive definite kernel $k_{Y}$, the \emph{Stein
witness function }\citep{ChwStrGre2016,LiuLeeJor2016} $\bg^{q,r}\colon\mathcal{X}\to\mathbb{R}^{d}$
between $q$ and $r$ is defined as $\bg^{q,r}(\boldsymbol{w}):=\mathbb{E}_{\boldsymbol{z}\sim r}\left[\bxi^{q}(\boldsymbol{z},\boldsymbol{w})\right]=(g_{1}^{q,r}(\boldsymbol{w}),\ldots,g_{d}^{q,r}(\boldsymbol{w}))^{\top},$
where $\bxi^{q}(\boldsymbol{z},\boldsymbol{w}):=k_{Y}(\boldsymbol{z},\boldsymbol{w})\nabla_{\boldsymbol{z}}\log q(\boldsymbol{z})+\nabla_{\boldsymbol{z}}k_{Y}(\boldsymbol{z},\boldsymbol{w})$.
 Under appropriate conditions (see \citet[Theorem 2.2]{ChwStrGre2016}
and \citet[Proposition 3.3]{LiuLeeJor2016}), it can be shown that
$\bg^{q,r}=\boldsymbol{0}$ (i.e., the zero function) if and only
if $q=r$. An implication of this result is that the deviation of
$\bg^{q,r}$ from the zero function can be used as a measure of mismatch
between $q$ and $r$. Different ways to characterize such deviation
have led to different measures of goodness of fit. 

The FSSD characterizes such deviation from $\boldsymbol{0}$ by evaluating
$\boldsymbol{g}^{q,r}$ at $J_{q}$ test locations. Formally, given
a set of test locations $W=\{\boldsymbol{w}_{j}\}_{j=1}^{J_{q}}$,
the squared FSSD is defined as $\mathrm{FSSD}_{q}^{2}(r):=\frac{1}{dJ_{q}}\sum_{j=1}^{J_{q}}\|\bg^{q,r}(\boldsymbol{w}_{j})\|_{2}^{2}:=\fssdq$
\citep{JitXuSzaFukGre2017}. Under appropriate conditions, it is known
that almost surely $\fssdq=0$ if and only if $q=r$. Using the notations
as in \citet{JitXuSzaFukGre2017}, one can write $\fssdq=\mathbb{E}_{\boldsymbol{z}\sim r}\mathbb{E}_{\boldsymbol{z}'\sim r}\Delta_{q}(\boldsymbol{z},\boldsymbol{z}')$
where $\Delta_{q}(\boldsymbol{z},\boldsymbol{z}'):=\btau_{q}^{\top}(\boldsymbol{z})\btau_{q}(\boldsymbol{z}')$,
$\btau_{q}(\boldsymbol{z}):=\mathrm{vec}(\boldsymbol{\Xi}^{q}(\boldsymbol{z}))\in\mathbb{R}^{dJ_{q}}$,
$\mathrm{vec}(\boldsymbol{M})$ concatenates columns of $\boldsymbol{M}$
into a column vector, and $\boldsymbol{\Xi}^{q}(\boldsymbol{z})\in\mathbb{R}^{d\times J_{q}}$
is defined such that $[\boldsymbol{\Xi}^{q}(\boldsymbol{z})]_{i,j}:=\xi_{i}^{q}(\boldsymbol{z},\boldsymbol{w}_{j})/\sqrt{dJ_{q}}$
for $i=1,\ldots,d$ and $j=1,\ldots,J_{q}$. Equivalently, $\fssdq=\|\bmu_{q}\|_{2}^{2}$
where $\bmu_{q}:=\mathbb{E}_{\boldsymbol{z}\sim r}[\btau_{q}(\boldsymbol{z})]$.
Similar to the UME statistic described previously, given a sample
$Z_{n}=\{\boldsymbol{z}_{i}\}_{i=1}^{n}\sim r$, an unbiased estimator
of $\fssdq$, denoted by $\fssdhq$ can be straightforwardly written
as a second-order U-statistic, which can be computed in $\mathcal{O}(J_{q}n)$
time. It was shown in \citet{JitXuSzaFukGre2017} that the test locations
$W$ can be chosen by maximizing the test power of the goodness-of-fit
test proposing $H_{0}:q=r$ against $H_{1}:q\neq r$, using $\fssdhq$
as the statistic. We note that, unlike UME, $\fssdhq$ requires access
to the density $q$. Another way to characterize the deviation of
$\boldsymbol{g}^{q,r}$ from the zero function is to use the norm
in the reproducing kernel Hilbert space (RKHS) that contains $\boldsymbol{g}^{q,r}$.
This measure is known as the Kernel Stein Discrepancy having a runtime
complexity of $\mathcal{O}(n^{2})$ \citep{ChwStrGre2016,LiuLeeJor2016,GorMac2015}.\vspace{-2mm}

\section{Proposal: $\protect\relume$ and $\protect\relfssd$ Tests}

\vspace{-2mm}\label{sec:new_mctests}\textbf{Relative UME ($\relume$)
}Our first proposed relative goodness-of-fit test based on UME tests
$H_{0}\colon U^{2}(P,R)\le U^{2}(Q,R)$ versus $H_{1}\colon U^{2}(P,R)>U^{2}(Q,R)$.
The test uses $\sqrt{n}\hat{S}_{n}^{U}=\sqrt{n}(\umehpr-\umehqr)$
as the statistic, and rejects $H_{0}$ when it is larger than the
threshold $T_{\alpha}$. The threshold is given by the $(1-\alpha)$-quantile
of the asymptotic distribution of $\sqrt{n}\hat{S}_{n}^{U}$ when
$H_{0}$ holds i.e., the null distribution, and the pre-chosen $\alpha$
is the significance level. It is well-known that this choice for the
threshold asymptotically controls the false rejection rate to be bounded
above by $\alpha$ yielding a level-$\alpha$ test \citep[Definition 8.3.6]{CasBer2002}.
In the full generality of $\relume$, two sets of test locations can
be used: $V=\{\boldsymbol{v}_{j}\}_{j=1}^{J_{p}}$ for computing $\umehpr$,
and $W=\{\boldsymbol{w}_{j}\}_{j=1}^{J_{q}}$ for $\umehqr$. The
feature function for $\umehpr$ is denoted by $\psi_{V}(\boldsymbol{x}):=\frac{1}{\sqrt{J_{p}}}\left(k_{X}(\boldsymbol{x},\boldsymbol{v}_{1}),\ldots,k_{X}(\boldsymbol{x},\boldsymbol{v}_{J_{p}})\right)^{\top}\in\mathbb{R}^{J_{p}},$
for some kernel $k_{X}$ which can be different from $k_{Y}$ used
in $\psi_{W}$. The asymptotic distribution of the statistic is stated
in Theorem \ref{thm:normal_sume}.

%----------------------------------
%---------- Theorem -------------
\begin{restatable}[Asymptotic  distribution of $\hat{S}^U_n$]{thm}{normalsume}

\label{thm:normal_sume}Define $\covqw:=\mathrm{cov}_{\boldsymbol{y}\sim Q}[\psi_{W}(\boldsymbol{y}),\psi_{W}(\boldsymbol{y})]$,
$\covpv:=\mathrm{cov}_{\boldsymbol{x}\sim P}[\psi_{V}(\boldsymbol{x}),\psi_{V}(\boldsymbol{x})]$,
and $C_{VW}^{R}:=\mathrm{cov}_{\boldsymbol{z}\sim R}[\psi_{V}(\boldsymbol{z}),\psi_{W}(\boldsymbol{z})]\in\mathbb{R}^{J_{p}\times J_{q}}$.
Let $S^{U}:=\umepr-\umeqr$, and $\boldsymbol{M}:=\left(\begin{array}{cc}
\psipv-\psirv & \boldsymbol{0}\\
\boldsymbol{0} & \psiqw-\psirw
\end{array}\right)\in\mathbb{R}^{(J_{p}+J_{q})\times2}$. Assume that \uline{1)} $P,Q$ and $R$ are all distinct, \uline{2)}
$(k_{X},V)$ are chosen such that $U_{P}^{2}>0$, and $(k_{Y},W)$
are chosen such that $U_{Q}^{2}>0$, \uline{3)} $\left(\begin{array}{cc}
\zeta_{P}^{2} & \zeta_{PQ}\\
\zeta_{PQ} & \zeta_{Q}^{2}
\end{array}\right):=\boldsymbol{M}^{\top}\left(\begin{array}{cc}
\covpv+\covrv & C_{VW}^{R}\\
(C_{VW}^{R})^{\top} & \covqw+\covrw
\end{array}\right)\boldsymbol{M}$ is positive definite.    Then, $\sqrt{n}\left(\widehat{S}_{n}^{U}-S^{U}\right)\stackrel{d}{\to}\mathcal{N}\left(0,4(\zeta_{P}^{2}-2\zeta_{PQ}+\zeta_{Q}^{2})\right)$

\end{restatable}
%----------------------------------
%----------------------------------

A proof of Theorem \ref{thm:normal_sume} can be found in Section
\ref{sec:proof_normal_sume} (appendix). Let $\nu:=4(\zeta_{P}^{2}-2\zeta_{PQ}+\zeta_{Q}^{2})$.
Theorem \ref{thm:normal_sume} states that the asymptotic distribution
of $\hat{S}_{n}^{U}$ is normal with the mean given by $S^{U}:=\umepr-\umeqr$.
It follows that under $H_{0}$, $S^{U}\le0$ and the $(1-\alpha)$-quantile
is $S^{U}+\sqrt{\nu}\Phi^{-1}(1-\alpha)$ where $\Phi^{-1}$ is the
quantile function of the standard normal distribution. Since $S^{U}$
is unknown in practice, we therefore adjust it to be $\sqrt{\nu}\Phi^{-1}(1-\alpha)$,
and use it as the test threshold $T_{\alpha}$. The adjusted threshold
can be estimated easily by replacing $\nu$ with $\hat{\nu}_{n}$,
a consistent estimate based on samples. It can be shown that the test
with the adjusted threshold is still level-$\alpha$ (more conservative
in rejecting $H_{0}$). We note that the same approach of adjusting
the threshold is used in $\relmmd$ \citep{BouBelBlaAntGre2015}. 

\textbf{Better Fit of $Q$ in Terms of $W$.} When specifying $V$
and $W$, the model comparison is done by comparing the goodness of
fit of $P$ (to $R$) as measured in the regions specified by $V$
to the goodness of fit of $Q$ as measured in the regions specified
by $W$. By specifying $V$ and setting $W=V$, testing with $\relume$
is equivalent to posing the question ``\emph{Does $Q$ fit to the
data better than $P$ does, as measured in the regions of $V$?}''
For instance, the observed sample from $R$ might contain smiling
and non-smiling faces, and $P,Q$ are candidate generative models
for face images. If we are interested in checking the relative fit
in the regions of smiling faces, $V$ can be a set of smiling faces.
In the followings, we will assume $V=W$ and $k:=k_{X}=k_{Y}$ for
interpretability. Investigating the general case without these constraints
will be an interesting topic of future study. Importantly we emphasize
that test results are always conditioned on the specified $V$. To
be precise, let $U_{V}^{2}$ be the squared UME statistic defined
by $V$. It is entirely realistic that the test rejects $H_{0}$ in
favor of $H_{1}\colon U_{V_{1}}^{2}(P,R)>U_{V_{1}}^{2}(Q,R)$ (i.e.,
$Q$ fits better) for some $V_{1}$, and also rejects $H_{0}$ in
favor of the opposite alternative $H_{1}\colon U_{V_{2}}^{2}(Q,R)>U_{V_{2}}^{2}(P,R)$
(i.e., $P$ fits better) for another setting of $V_{2}$. This is
because the regions in which the model comparison takes place are
different in the two cases. Although not discussed in \citet{BouBelBlaAntGre2015},
the same behaviour can be observed for $\relmmd$ i.e., test results
are conditioned on the choice of kernel.

In some cases, it is not known in advance what features are better
represented by one model vs another, and it becomes necessary to learn
these features from the model outputs. In this case, we propose setting
$V$ to contain the locations which maximize the probability that
the test can detect the better fit of $Q$, as measured at the locations.
Following the same principle as in \citet{GreSejStrBalPon2012,SutTunStrDeRam2016,JitSzaChwGre2016,JitSzaGre2017,JitXuSzaFukGre2017},
this goal can be achieved by finding $(k,V)$ which maximize the test
power, while ensuring that the test is level-$\alpha$. By Theorem
\ref{thm:normal_sume}, for large $n$ the test power $\mathbb{P}\left(\sqrt{n}\hat{S}_{n}^{U}>T_{\alpha}\right)$
is approximately $\Phi\left(\frac{\sqrt{n}S^{U}-T_{\alpha}}{\sqrt{\nu}}\right)=\Phi\left(\sqrt{n}\frac{S^{U}}{\sqrt{\nu}}-\sqrt{\frac{\hat{\nu}_{n}}{\nu}}\Phi^{-1}(1-\alpha)\right)$.
Under $H_{1}$, $S^{U}>0$. For large $n$, $\Phi^{-1}(1-\alpha)\sqrt{\hat{\nu}_{n}}/\sqrt{\nu}$
approaches a constant, and $\sqrt{n}S^{U}/\sqrt{\nu}$ dominates.
It follows that, for large $n$, $(k^{*},V^{*})=\arg\max_{(k,V)}\mathbb{P}\left(\sqrt{n}\hat{S}_{n}^{U}>T_{\alpha}\right)\approx\arg\max_{(k,V)}S^{U}/\sqrt{\nu}$.
We can thus use $\hat{S}_{n}^{U}/(\gamma+\sqrt{\hat{\nu}_{n}})$ as
an estimate of the \emph{power criterion} objective $S^{U}/\sqrt{\nu}$
for the test power, where $\gamma>0$ is a small regularization parameter
added to promote numerical stability following \citet[p.\ 5]{JitXuSzaFukGre2017}.
To control the false rejection rate, the maximization is carried out
on held-out training data which are independent of the data used for
testing. In the experiments (Section \ref{sec:experiments}), we hold
out 20\% of the data for the optimization. A unique consequence of
this procedure is that we obtain optimized $V^{*}$ which indicates
where $Q$ fits significantly better than $P$. We note that this
interpretation only holds if the test, using the optimized hyperparameters
$(k^{*},V^{*})$, decides to reject $H_{0}$. The optimized locations
may not be interpretable if the test fails to reject $H_{0}$.

\textbf{Relative FSSD ($\relfssd$)} The proposed $\relfssd$ tests
$H_{0}\colon\fssdp\le\fssdq$ versus $H_{1}\colon\fssdp>\fssdq$.
The test statistic is $\sqrt{n}\hat{S}_{n}^{F}:=\sqrt{n}(\fssdhp-\fssdhq)$.
We note that the feature functions $\boldsymbol{\tau}_{p}$ (for $\fssdp$)
and $\boldsymbol{\tau}_{q}$ (for $\fssdq$) depend on $(k_{X},V)$
and $(k_{Y},W)$ respectively, and play the same role as the feature
functions $\psi_{V}$ and $\psi_{W}$ of the UME statistic. Due to
space limitations, we only state the salient facts of $\relfssd$.
The rest of the derivations closely follow $\relume$. These include
the interpretation that the relative fit is measured at the specified
locations given in $V$ and $W$, and the derivation of $\relfssd$'s
power criterion (which can be derived using the asymptotic distribution
of $\hat{S}_{n}^{F}$ given in Theorem \ref{thm:normal_sfssd}, following
the same line of reasoning as in the case of $\relume$). A major
difference is that $\relfssd$ requires explicit (gradients of the
log) density functions of the two models, allowing it to gain structural
information of the models that may not be as easily observed in finite
samples. We next state the asymptotic distribution of the statistic
(Theorem \ref{thm:normal_sfssd}), which is needed for obtaining the
threshold and for deriving the power criterion. The proof closely
follows the proof of Theorem \ref{thm:normal_sume}, and is omitted.

\begin{thm}[Asymptotic  distribution of $\hat{S}_{n}^{F}$]
 \label{thm:normal_sfssd} Define $S^{F}:=F_{p}^{2}-F_{q}^{2}$.
Let $\boldsymbol{\Sigma}^{ss'}:=\mathrm{cov}_{\boldsymbol{z}\sim r}[\btau_{s}(\boldsymbol{z}),\btau_{s'}(\boldsymbol{z})]$
for $s,s'\in\{p,q\}$ so that $\boldsymbol{\Sigma}^{pq}\in\mathbb{R}^{dJ_{p}\times dJ_{q}}$,
$\boldsymbol{\Sigma}^{qp}:=\left(\boldsymbol{\Sigma}^{pq}\right)^{\top}$,
$\boldsymbol{\Sigma}^{pp}=\boldsymbol{\Sigma}^{p}\in\mathbb{R}^{dJ_{p}\times dJ_{p}}$,
and $\boldsymbol{\Sigma}^{qq}=\boldsymbol{\Sigma}^{q}\in\mathbb{R}^{dJ_{q}\times dJ_{q}}$.
Assume that \uline{1)} $p,q,$ and $r$ are all distinct, \uline{2)}
$(k_{X},V)$ are chosen such that $\fssdp>0$, and $(k_{Y},W)$ are
chosen such that $\fssdq>0$, \uline{3)} $\left(\begin{array}{cc}
\sigma_{p}^{2} & \sigma_{pq}\\
\sigma_{pq} & \sigma_{q}^{2}
\end{array}\right):=\left(\begin{array}{cc}
\bmu_{p}^{\top}\boldsymbol{\Sigma}^{p}\bmu_{p} & \bmu_{p}^{\top}\boldsymbol{\Sigma}^{pq}\bmu_{q}\\
\bmu_{p}^{\top}\boldsymbol{\Sigma}^{pq}\bmu_{q} & \bmu_{q}^{\top}\boldsymbol{\Sigma}^{q}\bmu_{q}
\end{array}\right)$ is positive definite. Then, $\sqrt{n}\left(\widehat{S}_{n}^{F}-S^{F}\right)\stackrel{d}{\to}\mathcal{N}\left(0,4(\sigma_{p}^{2}-2\sigma_{pq}+\sigma_{q}^{2})\right)$.
\end{thm}

\section{Experiments}

\vspace{-2mm}\label{sec:experiments} In this section, we demonstrate
the two proposed tests on both toy and real problems. We start with
an illustration of the behaviors of $\relume$ and $\relfssd$'s power
criteria using simple one-dimensional problems. In the second experiment,
we examine the test powers of the two proposed tests using three toy
problems. In the third experiment, we compare two hypothetical generative
models on the CIFAR-10 dataset \citep{KriHin2009} and demonstrate
that the learned test locations (images) can clearly indicate the
types of images that are better modeled by one of the two candidate
models.  In the last two experiments, we consider the problem of
determining the relative goodness of fit of two given Generative Adversarial
Networks (GANs) \citep{GooPouMirXuWar2014}. Code to reproduce all
the results is available at \url{https://github.com/wittawatj/kernel-mod}.

%--- define the titles of subsections ---
\newcommand{\exsehist}{Informative Power Objective}  
\newcommand{\exsecontopt}{Learning Unseen, Distinguishing Faces}  
\newcommand{\exsepowtoy}{Test Powers on Toy Problems}
\newcommand{\exseganrej}{Testing GAN Models}  
\newcommand{\exseganfeature}{Examining GAN Training}  

\textbf{1. Illustration of $\relume$ and $\relfssd$ Power Criteria}
We consider $k=k_{X}=k_{Y}$ to be a Gaussian kernel, and set $V=W=\{\boldsymbol{v}\}$
(one test location). The power criterion of $\relume$ as a function
of $\boldsymbol{v}$ can be written as $\frac{1}{2}\frac{\mathrm{wit}_{P,R}^{2}(\boldsymbol{v})-\mathrm{wit}_{Q,R}^{2}(\boldsymbol{v})}{(\zeta_{P}^{2}(\boldsymbol{v})-2\zeta_{PQ}(\boldsymbol{v})+\zeta_{Q}^{2}(\boldsymbol{v}))^{1/2}}$
where $\mathrm{wit}(\cdot)$ is the MMD witness function (see Section
\ref{sec:ume}), and we explicitly indicate the dependency on $\boldsymbol{v}$.
To illustrate, we consider two Gaussian models $p,q$ with different
means but the same variance, and set $r$ to be a mixture of $p$
and $q$. Figure \ref{fig:mmd_wit_isomix} shows that when each component
in $r$ has the same mixing proportion, the power criterion of $\relume$
is a zero function indicating that $p$ and $q$ have the same goodness
of fit to $r$ everywhere. To understand this, notice that at the
left mode of $r$, $p$ has excessive probability mass (compared to
$r$), while $q$ has almost no mass at all. Both models are thus
wrong at the left mode of $r$. However, since the extra probability
mass of $p$ is equal to the missing mass of $q$, $\relume$ considers
$p$ and $q$ as having the same goodness of fit. In Figure \ref{fig:mmd_wit_skewmix},
the left mode of $r$ now has a mixing proportion of only 30\%, and
$r$ more closely matches $q$. The power criterion is thus positive
at the left mode indicating that $q$ has a better fit.

The power criterion of $\relfssd$ indicates that $q$ fits better
at the right mode of $r$ in the case of equal mixing proportion (see
Figure \ref{fig:stein_wit_isomix}). In one dimension, the Stein witness
function $\boldsymbol{g}^{q,r}$ (defined in Section \ref{sec:fssd})
can be written as $g^{q,r}(w)=\mathbb{E}_{z\sim r}\left[k_{Y}(z,w)\nabla_{z}\left(\log q(z)-\log r(z)\right)\right]$,
which is the expectation under $r$ of the difference in the derivative
log of $q$ and $r$, weighted by the kernel $k_{Y}$. The Stein witness
thus only captures the matching of the shapes of the two densities
(as given by the derivative log). Unlike the MMD witness, the Stein
witness is insensitive to the mismatch of probability masses i.e.,
it is independent of the normalizer of $q$. In Figure \ref{fig:stein_wit_isomix},
since the shape of $q$ and the shape of the right mode of $r$ match,
the Stein witness $g^{q,r}$ (dashed blue curve) vanishes at the right
mode of $r$, indicating a good fit of $q$ in the region. The mismatch
between the shape of $q$ and the shape of $r$ at the left mode of
$r$ is what creates the peak of $g^{q,r}$. The same reasoning holds
for the Stein witness $g^{p,r}$. The power criterion of $\relfssd$,
which is given by $\frac{1}{2}\frac{g^{p,r}(w)^{2}-g^{q,r}(w)^{2}}{(\sigma_{p}^{2}(w)-2\sigma_{pq}(w)+\sigma_{q}^{2}(w))^{1/2}}$,
is thus positive at the right mode of $r$ (shapes of $q$ and $r$
matched there), and negative at the left mode of $r$ (shapes of $p$
and $r$ matched there). To summarize, $\relume$ measures the relative
fit by checking the probability mass, while $\relfssd$ does so by
matching the shapes of the densities. 

\begin{figure}
\hspace{-11mm}\subfloat[$\protect\relume$\label{fig:mmd_wit_isomix}]{\includegraphics[bb=80bp 40bp 510bp 305bp,clip,width=0.24\textwidth]{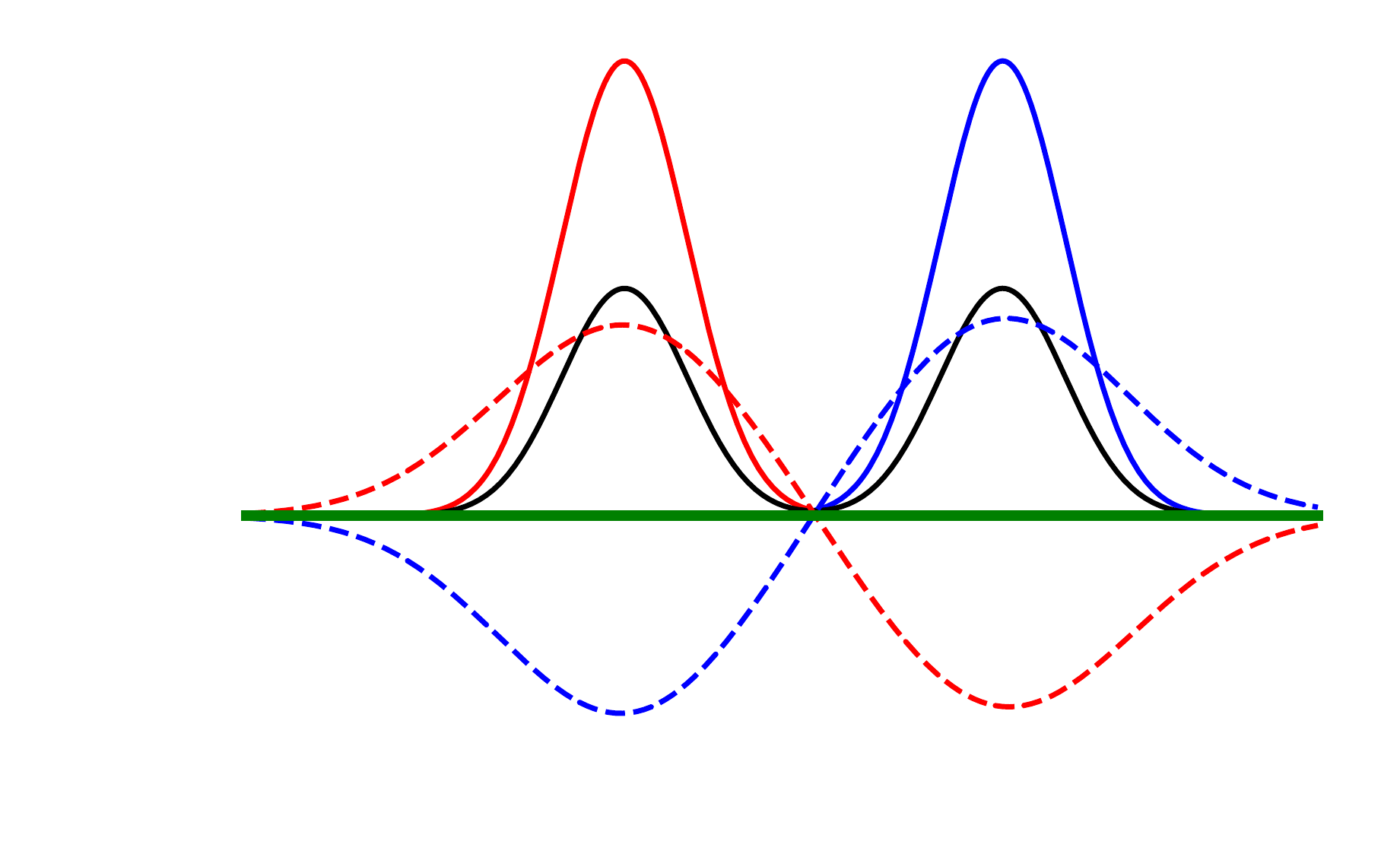}

}\hspace{-0mm}\subfloat[$\protect\relume$\label{fig:mmd_wit_skewmix}]{\includegraphics[bb=70bp 40bp 500bp 305bp,clip,width=0.25\textwidth]{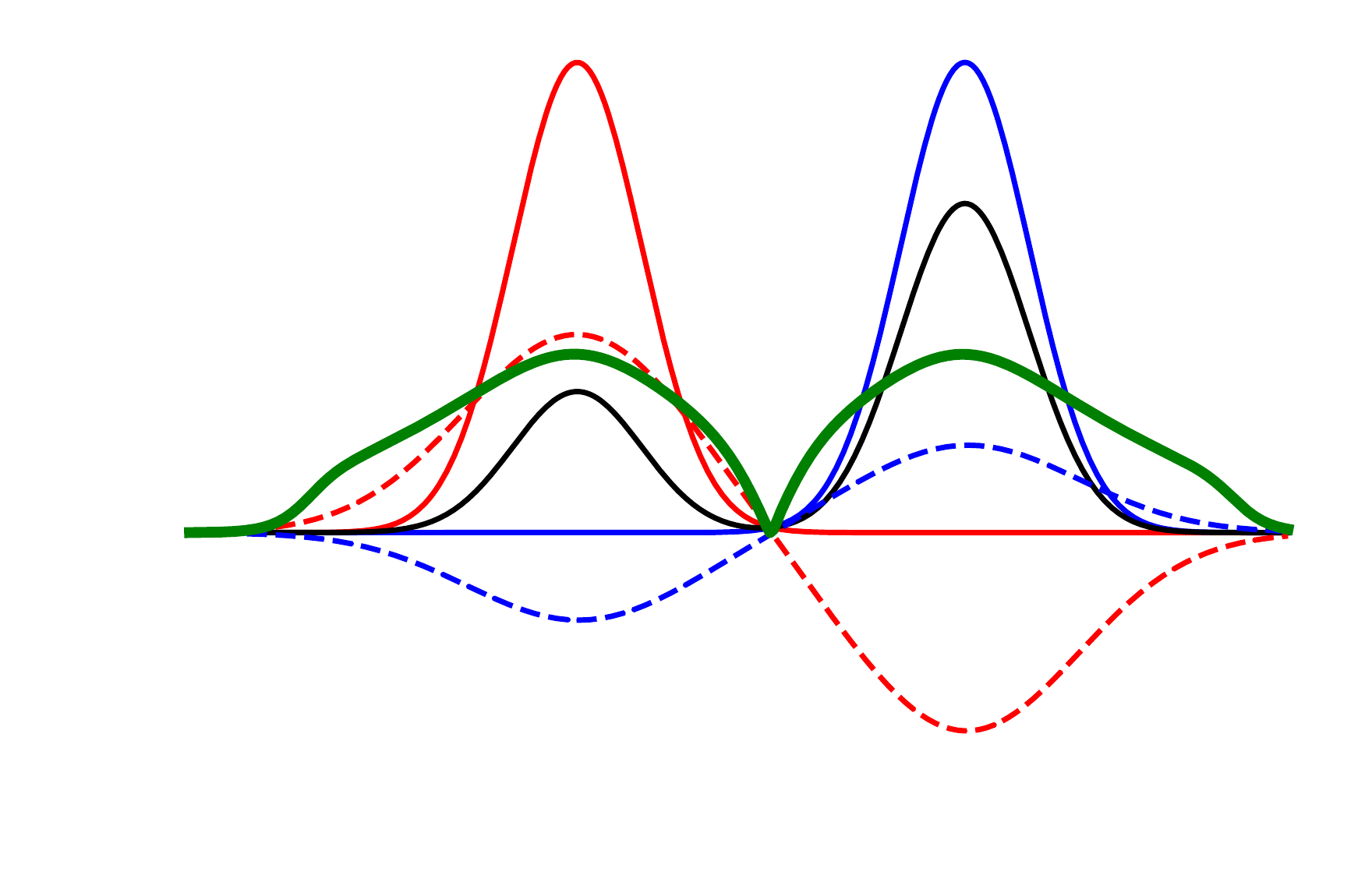} 

}\hspace{-2mm}\includegraphics[width=0.15\textwidth]{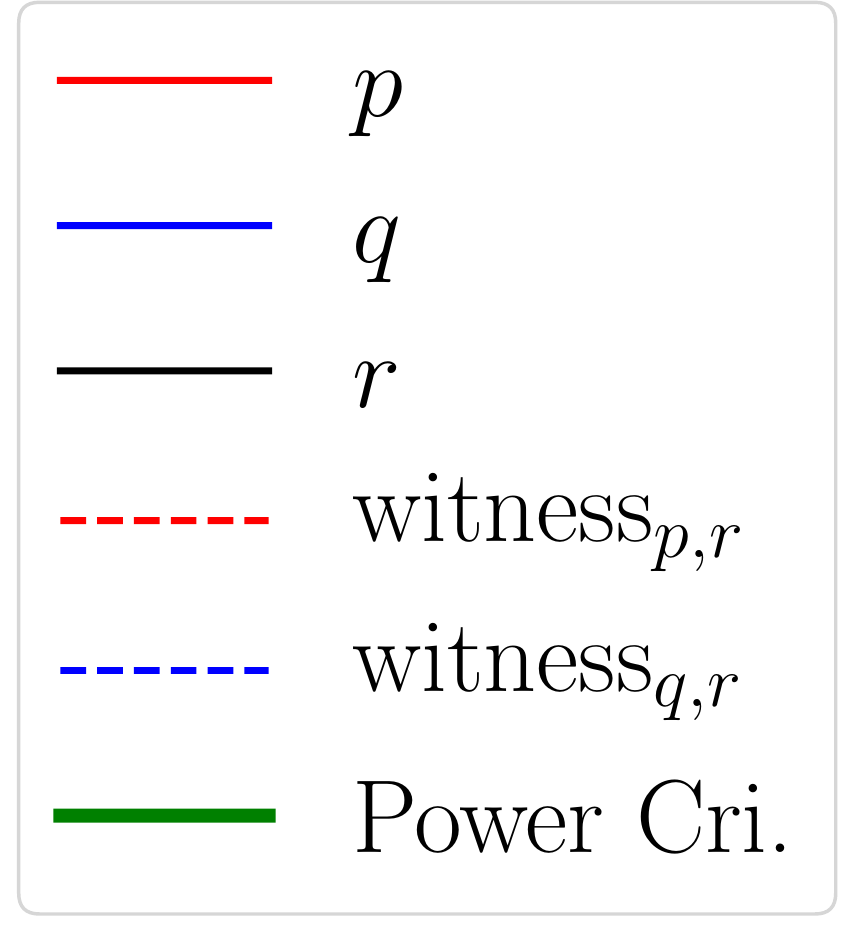}\subfloat[$\protect\relfssd$\label{fig:stein_wit_isomix}]{\includegraphics[bb=60bp 40bp 490bp 305bp,clip,width=0.24\textwidth]{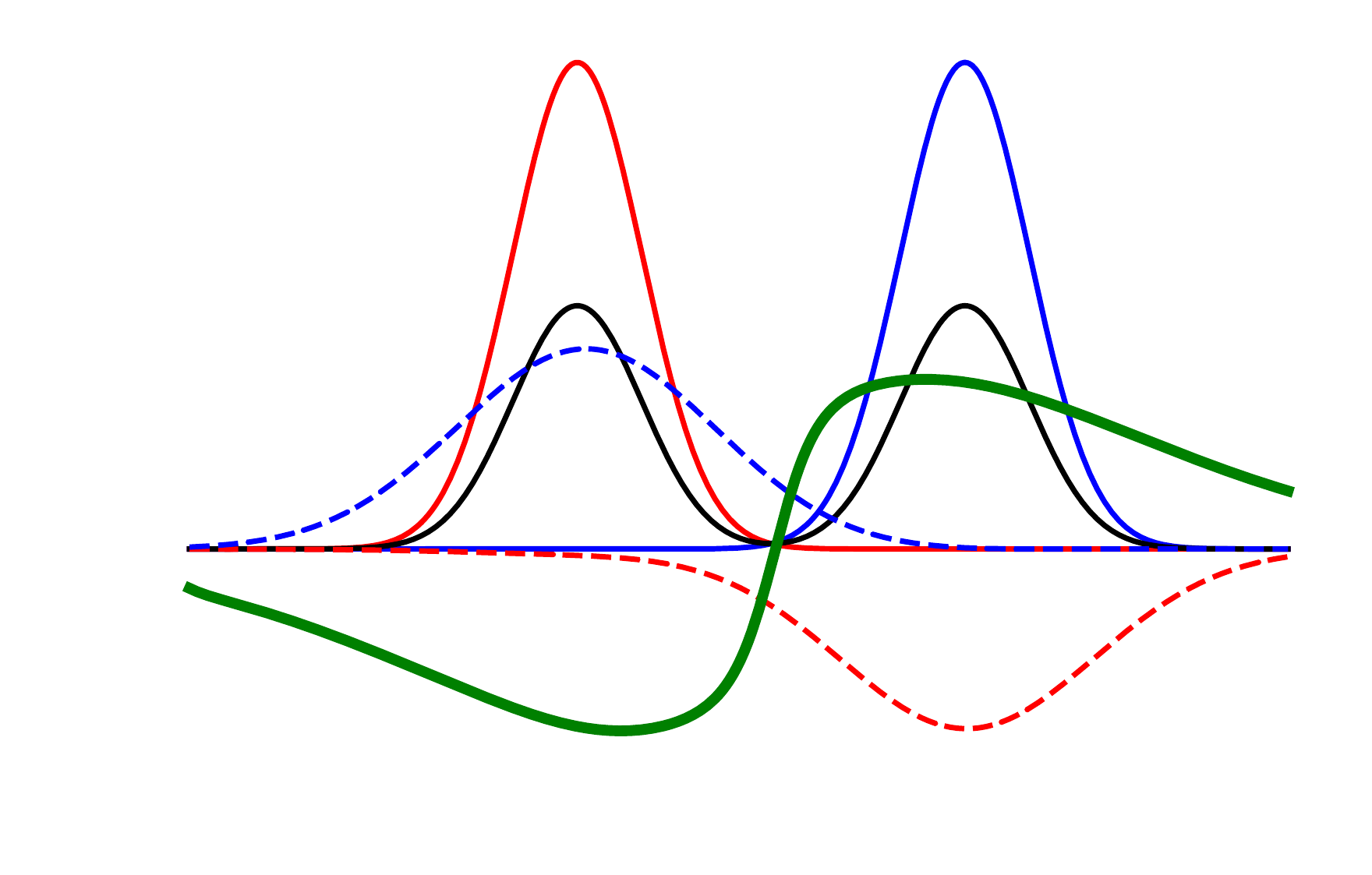}

}\subfloat[$\protect\relfssd$\label{fig:stein_wit_skewmix}]{\includegraphics[bb=60bp 40bp 490bp 305bp,clip,width=0.24\textwidth]{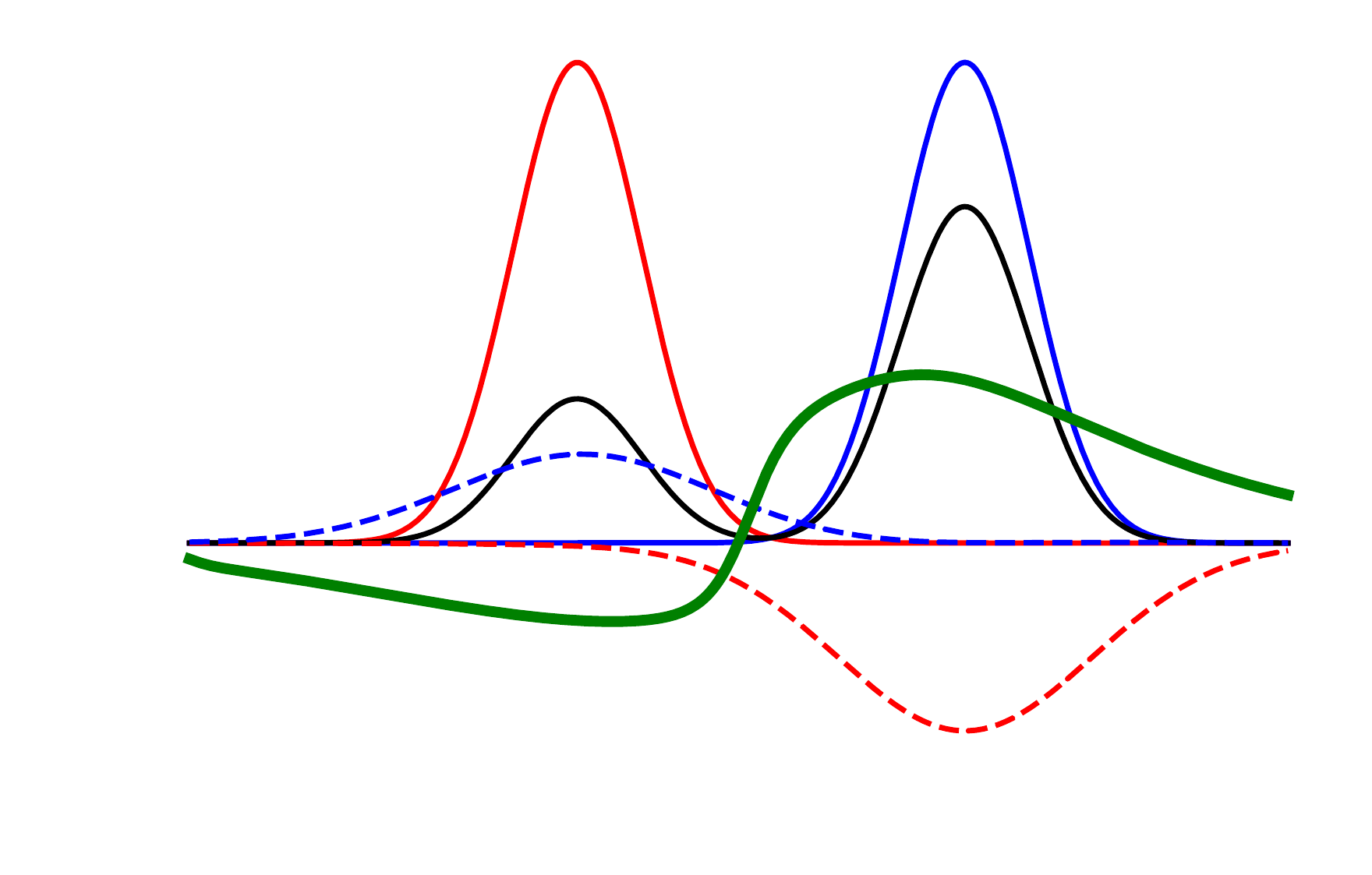}

}\caption{One-dimensional plots (in green) of $\protect\relume$'s power criteria
(in (a), (b)), and $\protect\relfssd$'s power criteria (in (c), (d)).
The dashed lines in (a), (b) indicate MMD's witness functions used
in $\protect\relume$, and the dashed lines in (c), (d) indicate FSSD's
Stein witness functions. \label{fig:witness_illustration} \vspace*{-4mm}}
\end{figure}

%--------------------------
\textbf{2. \exsepowtoy} The goal of this experiment is to investigate the rejection rates
of several variations of the two proposed tests. To this end, we study
three toy problems, each having its own characteristics. All the three
distributions in each problem have density functions to allow comparison
with $\relfssd$. 
\begin{enumerate}
\item \emph{Mean shift}:  All the three distributions are isotropic multivariate
normal distributions: $p=\mathcal{N}([0.5,0,\ldots,0],\boldsymbol{I}),q=\mathcal{N}([1,0,\ldots0],\boldsymbol{I}),$
and $r=\mathcal{N}(\boldsymbol{0},\boldsymbol{I})$, defined on $\mathbb{R}^{50}$.
The two candidates models $p$ and $q$ differ in the mean of the
first dimension. In this problem, the null hypothesis $H_{0}$ is
true since $p$ is closer to $r$.
\item \emph{Blobs}: Each distribution is given by a mixture of four Gaussian
distributions organized in a grid in $\mathbb{R}^{2}$. Samples from
$p,q$ and $r$ are shown in Figure \ref{fig:blobs_2d_samples}. In
this problem, $q$ is closer to $r$ than $p$ is i.e., $H_{1}$ is
true. One characteristic of this problem is that the difference between
$p$ and $q$ takes place in a small scale relative to the global
structure of the data. This problem was studied in \citet{GreSejStrBalPon2012,ChwRamSejGre2015}. 
\item \emph{RBM}: Each of the three distributions is given by a Gaussian
Bernoulli Restricted Boltzmann Machine (RBM) model with density function
$p'_{\boldsymbol{B},\boldsymbol{b},\boldsymbol{c}}(\boldsymbol{x})=\sum_{\boldsymbol{h}}p'_{\boldsymbol{B},\boldsymbol{b},\boldsymbol{c}}(\boldsymbol{x},\boldsymbol{h})$,
where $p'_{\boldsymbol{B},\boldsymbol{b},\boldsymbol{c}}(\boldsymbol{x},\boldsymbol{h}):=\frac{1}{Z}\exp\left(\boldsymbol{x}^{\top}\boldsymbol{B}\boldsymbol{h}+\boldsymbol{b}^{\top}\boldsymbol{x}+\boldsymbol{c}^{\top}\boldsymbol{h}-\frac{1}{2}\|\boldsymbol{x}\|^{2}\right)$,
$\boldsymbol{h}\in\{-1,1\}^{d_{h}}$ is a latent vector, $Z$ is the
normalizer, and $\boldsymbol{B},\boldsymbol{b},\boldsymbol{c}$ are
model parameters. Let $r(\boldsymbol{x}):=p'_{\boldsymbol{B},\boldsymbol{b},\boldsymbol{c}}(\boldsymbol{x}),p(\boldsymbol{x}):=p'_{\boldsymbol{B}^{p},\boldsymbol{b},\boldsymbol{c}}(\boldsymbol{x}),$
and $q(\boldsymbol{x}):=p'_{\boldsymbol{B}^{q},\boldsymbol{b},\boldsymbol{c}}(\boldsymbol{x})$.
Following a similar setting as in \citet{LiuLeeJor2016,JitXuSzaFukGre2017},
we set the parameters of the data generating density $r$ by uniformly
randomly setting entries of $\boldsymbol{B}$ to be from $\{-1,1\}$,
and drawing entries of $\boldsymbol{b}$ and $\boldsymbol{c}$ from
the standard normal distribution. Let $\boldsymbol{\delta}$ be a
matrix of the same size as $\boldsymbol{B}$ such that $\delta_{1,1}=1$
and all other entries are 0. We set $\boldsymbol{B}^{q}=\boldsymbol{B}+0.3\boldsymbol{\delta}$
and $\boldsymbol{B}^{p}=\boldsymbol{B}+\epsilon\boldsymbol{\delta}$,
where the perturbation constant $\epsilon$ is varied. We fix the
sample size $n$ to 2000. Perturbing only one entry of\textbf{ }$\boldsymbol{B}$
creates a problem in which the difference of distributions can be
difficult to detect. This serves as a challenging benchmark to measure
the sensitivity of statistical tests \citep{JitXuSzaFukGre2017}.
We set $d=20$ and $d_{h}=5$. 
\end{enumerate}
\begin{figure}
\centering
\includegraphics[width=0.98\linewidth]{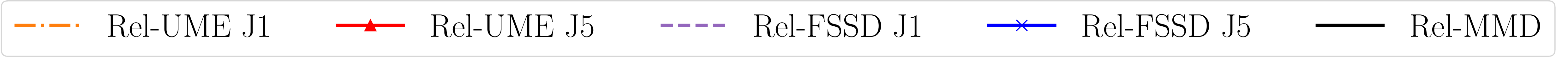} \\
\vspace{-3mm}
\subfloat[Mean shift. $d=50$. \label{fig:ex1_gauss_h0}]{
\includegraphics[width=0.24\linewidth]{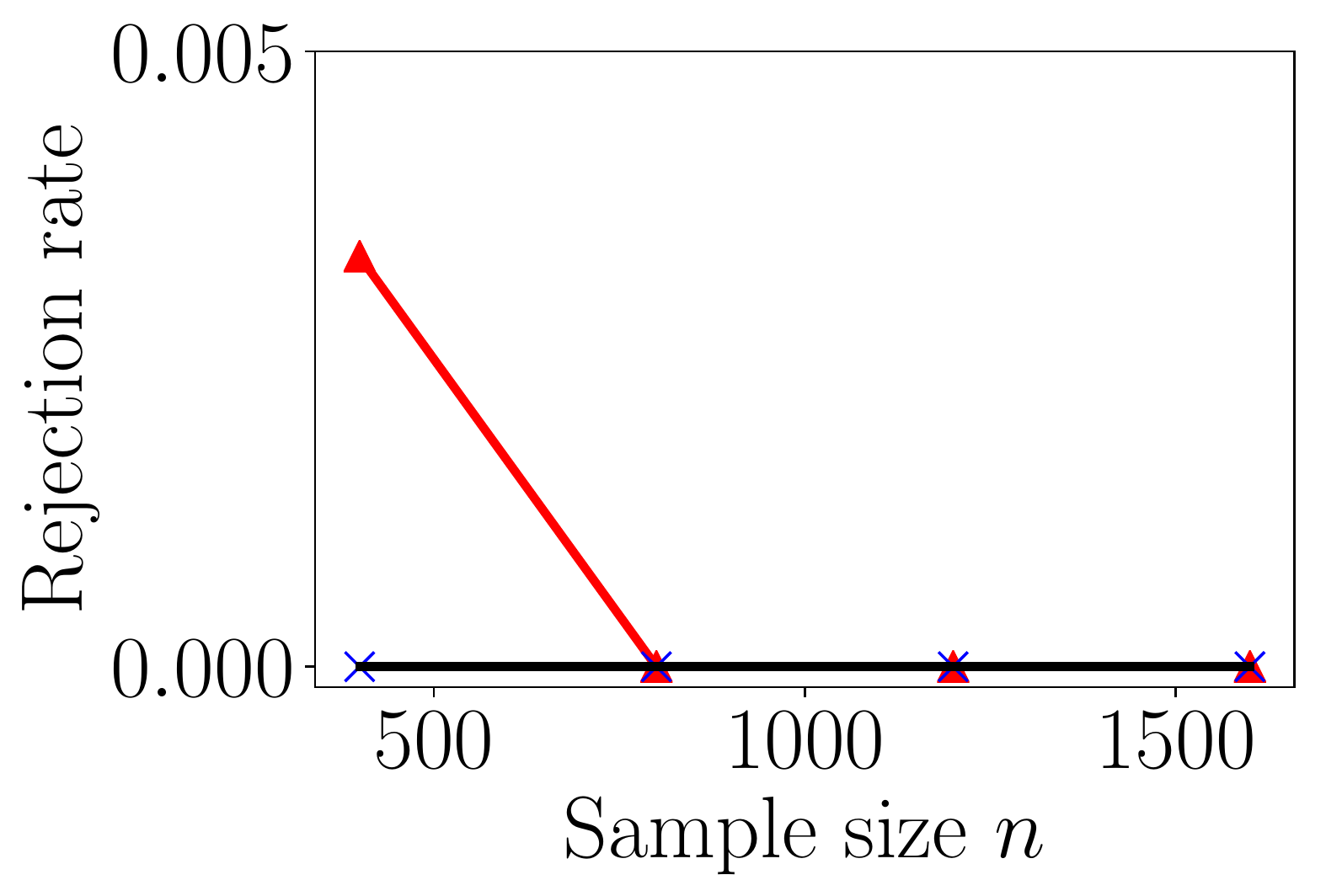}
}% 
\subfloat[Blobs. $d=2$. \label{fig:ex1_blobs}]{
\includegraphics[width=0.23\linewidth]{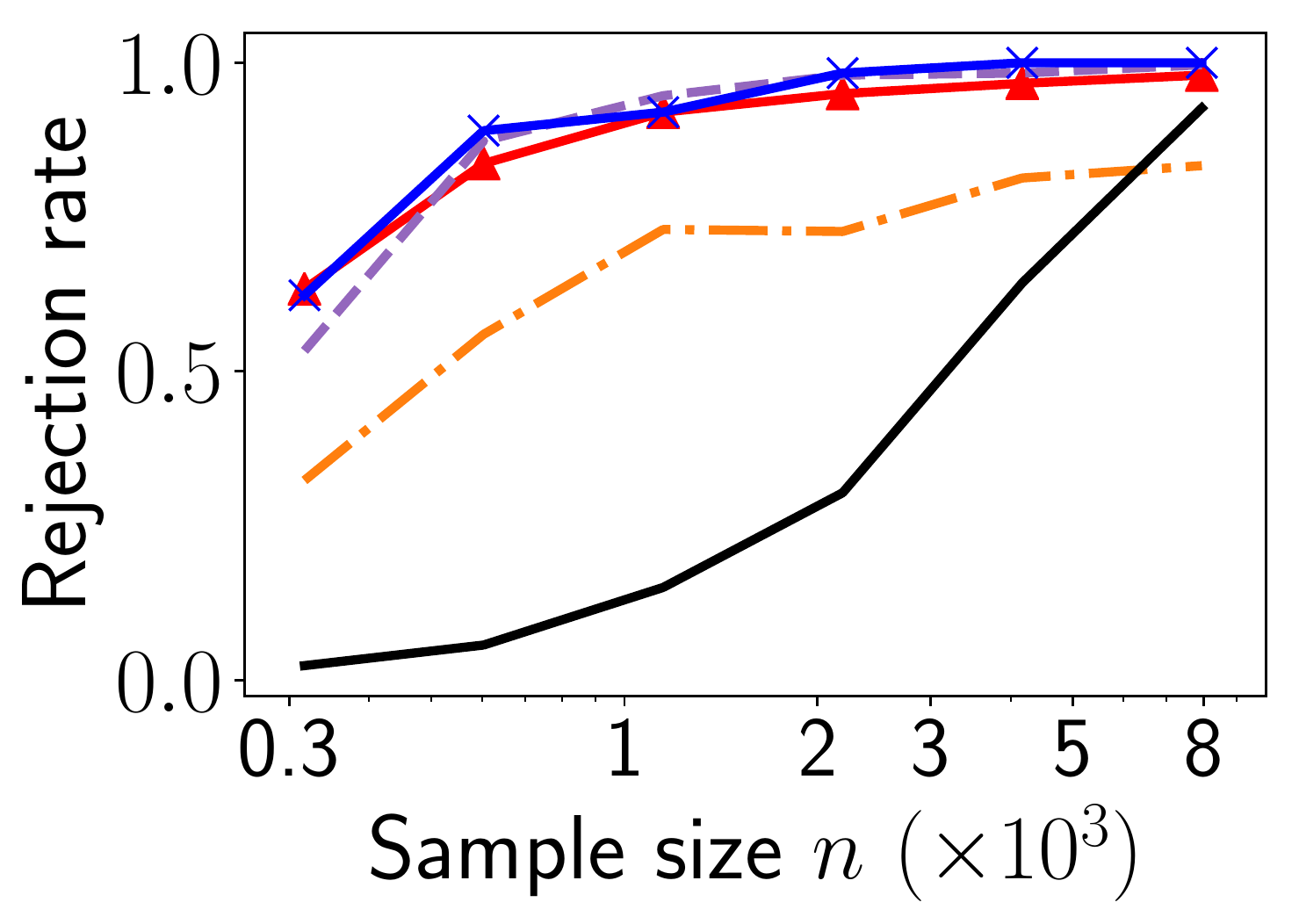} 
}% 
\subfloat[Blobs (Runtime) \label{fig:ex1_blobs_time}]{
\includegraphics[width=0.24\linewidth]{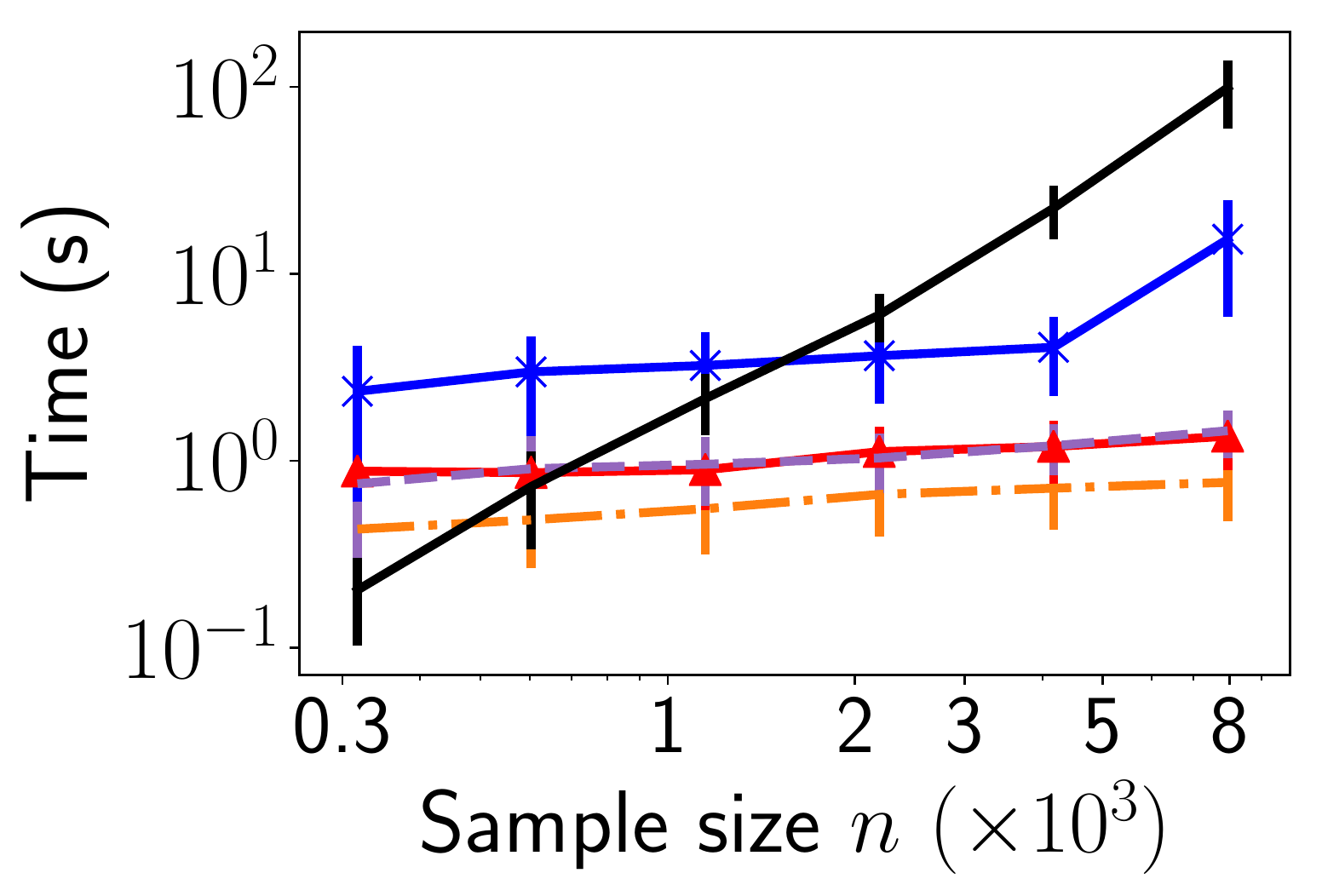}
}% \,\,
\subfloat[RBM. $d=20$\label{fig:ex2_rbm}]{
\includegraphics[width=0.23\linewidth]{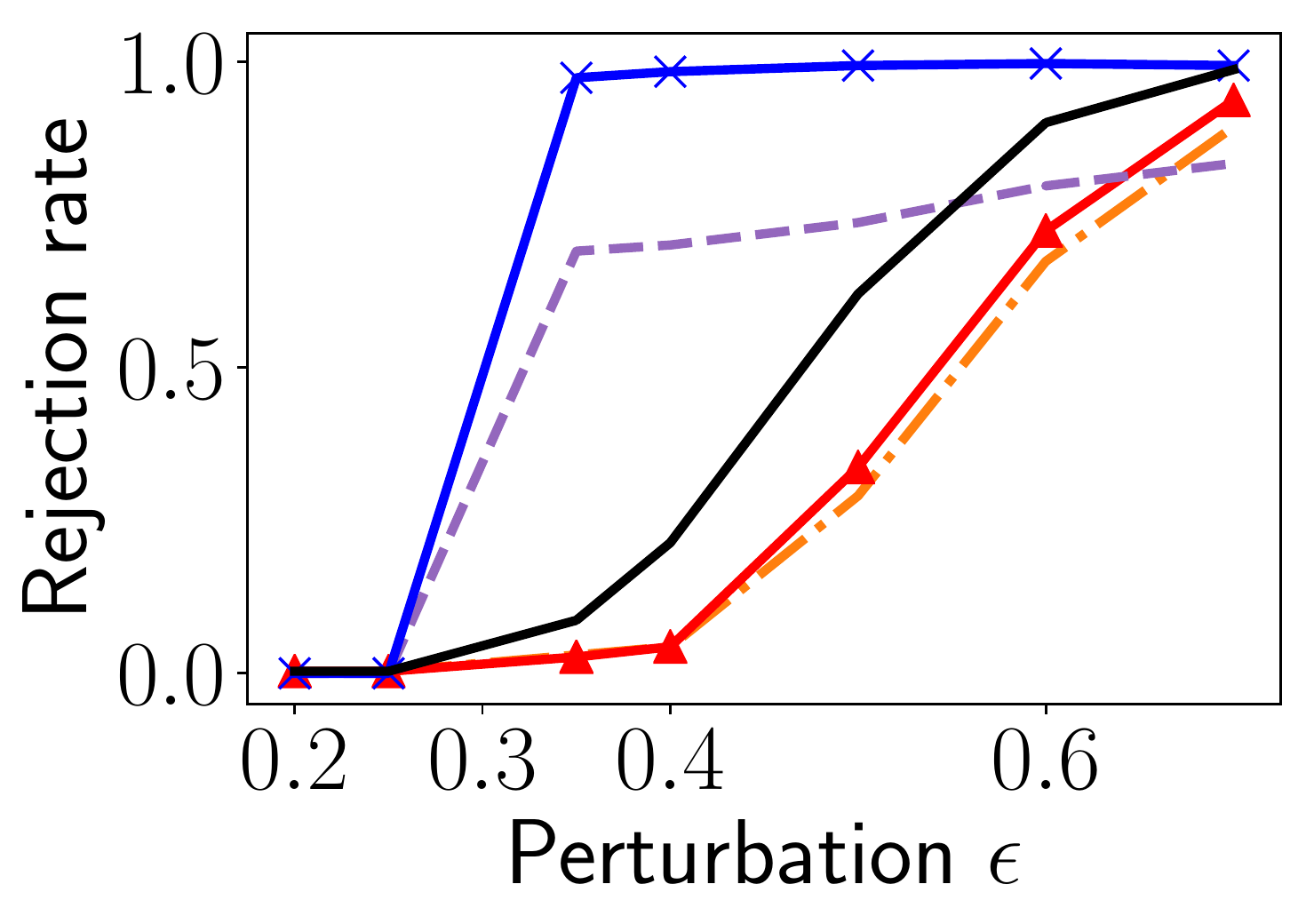}
}% 

\caption{(a), (b), (d) Rejection rates (estimated from 300 trials) of the five
tests with $\alpha=0.05$. In the RBM problem, $n=2000$. (c) Runtime
in seconds for one trial in the Blobs problem.  \vspace{-5mm}\label{fig:expr_test_powers}}
\end{figure}

We compare three kernel-based tests: $\relume$, $\relfssd$, and
$\relmmd$ (the relative MMD test of \citet{BouBelBlaAntGre2015}),
all using a Gaussian kernel. For $\relume$ and $\relfssd$ we set
$k_{X}=k_{Y}=k$, where the the Gaussian width of $k$, and the test
locations are chosen by maximizing their respective power criteria
described in Section \ref{sec:new_mctests} on 20\% of the data. The
optimization procedure is described in Section \ref{sec:test_locs_opt}
(appendix).  Following \citet{BouBelBlaAntGre2015}, the Gaussian
width of $\relmmd$ is chosen by the median heuristic as implemented
in the code by the authors. In the RBM problem, all problem parameters
$\boldsymbol{B},\boldsymbol{b},$ and $\boldsymbol{c}$ are drawn
only once and fixed. Only the samples vary across trials.  

\begin{figure}
\begin{centering}
\hspace{-2mm}
\newcommand{\psize}{5cm}\subfloat[Power Criterion\label{fig:cifar10_hist_ume_powcri}]{\includegraphics[width=0.24\textwidth]{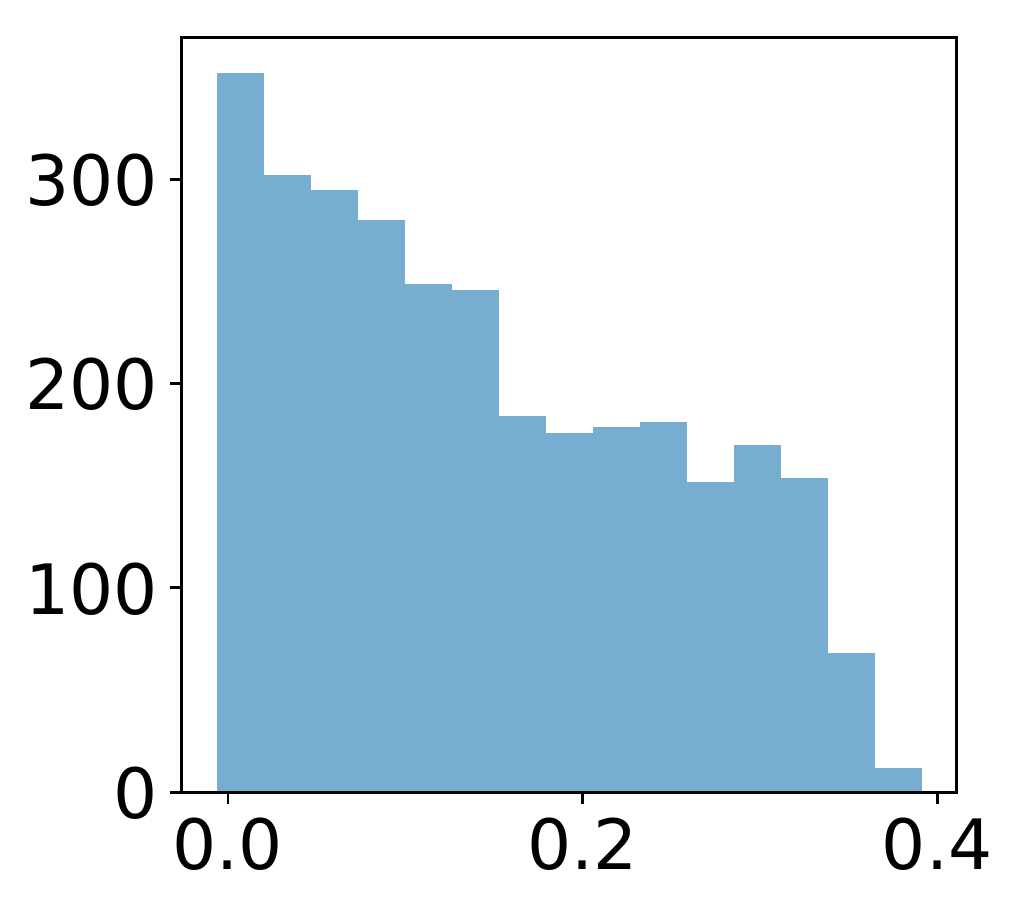}

}\,\subfloat[Sorted in ascending order\label{fig:cifar10_ume_powcri_0ascending}]{\includegraphics[width=\psize]{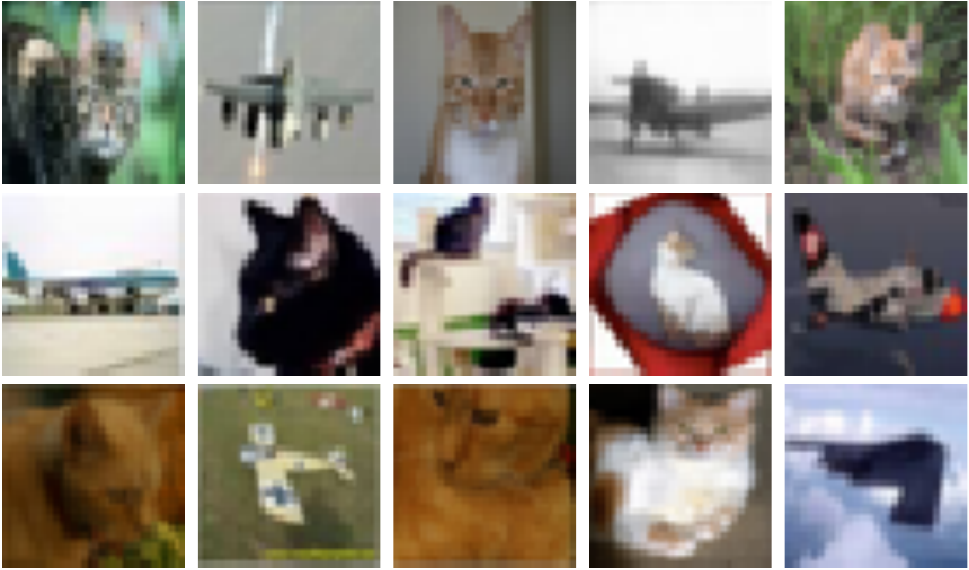}%

}\hspace{5mm}\subfloat[Sorted in descending order\label{fig:cifar10_ume_powcri_descending}]{\includegraphics[width=\psize]{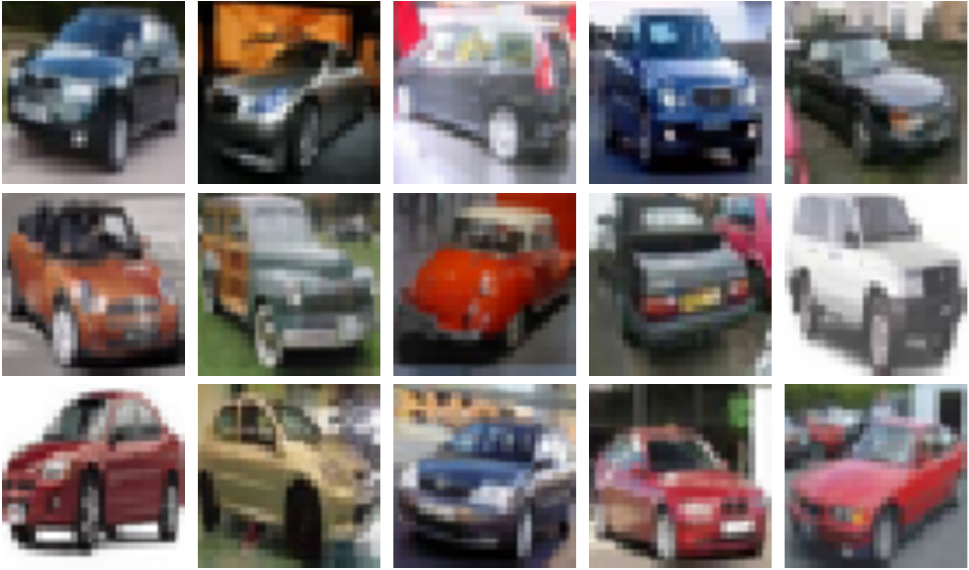}%

}
\par\end{centering}
\centering{}\caption{$P=$ \{airplane, cat\}, $Q=$ \{automobile, cat\}, and $R=$ \{automobile,
cat\}. (a) Histogram of $\protect\relume$ power criterion values.
(b), (c) Images as sorted by the criterion values in ascending and
descending orders, respectively. \vspace{-5mm}}
\end{figure}

\begin{wrapfigure}{l}{0.25\textwidth}   
\vspace{-5mm}
\begin{center}     
\includegraphics[width=3.5cm]{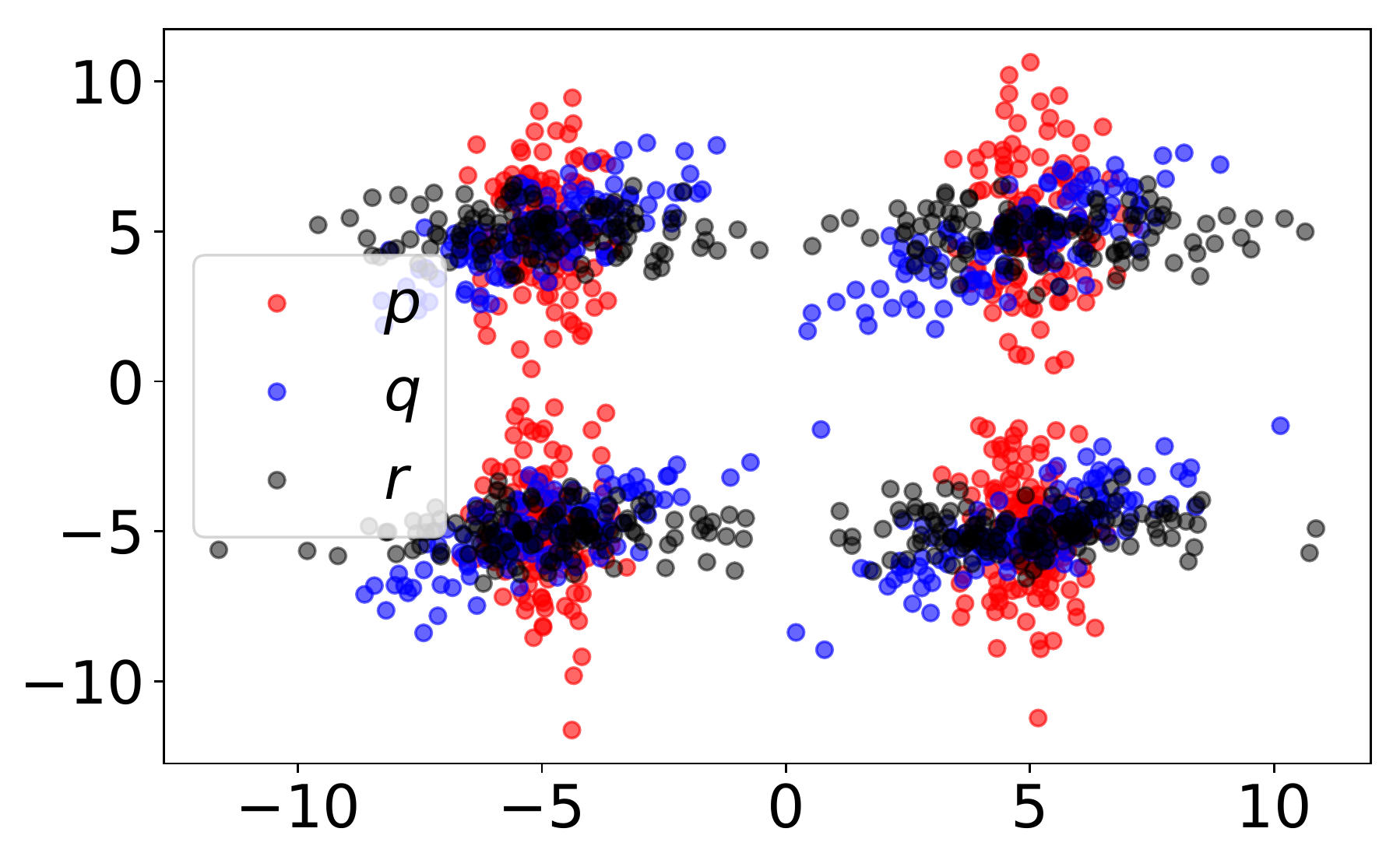}
\end{center}   
\caption{Blobs problem samples: \textcolor{red}{$p$}, \textcolor{blue}{$q$}, $r$.} 
\label{fig:blobs_2d_samples}
\vspace{-3mm}
\end{wrapfigure} Figure \ref{fig:expr_test_powers} shows the test powers of all the
tests. When $H_{0}$ holds, all tests have false rejection rates (type-I
errors) bounded above by $\alpha=0.05$ (Figure \ref{fig:ex1_gauss_h0}).
In the Blobs problem (Figure \ref{fig:ex1_blobs}), it can be seen
that $\relume$ achieves larger power at all sample sizes, compared
to $\relmmd$. Since the relative goodness of fit of $p$ and $q$
must be compared locally, the optimized test locations of $\relume$
are suitable for detecting such local differences. The poor performance
of $\relmmd$ is caused by unsuitable choices of the kernel bandwidth.
The bandwidth chosen by the median heuristic is only appropriate for
capturing the global length scale of the problem. It is thus too large
to capture small-scale differences. No existing work has proposed
a kernel selection procedure for $\relmmd$. Regarding the number
$J$ of test locations, we observe that changing $J$ from 1 to 5
drastically increases the test power of $\relume$, since more regions
characterizing the differences can be pinpointed. $\relmmd$ exhibits
a quadratic-time profile (Figure \ref{fig:ex1_blobs_time}) as a function
of $n$.

Figure \ref{fig:ex2_rbm} shows the rejection rates against the perturbation
strength $\epsilon$ in $p$ in the RBM problem. When $\epsilon\le0.3$,
$p$ is closer to $r$ than $q$ is (i.e., $H_{0}$ holds). We observe
that all the tests have well-controlled false rejection rates in this
case. At $\epsilon=0.35$, while $q$ is closer (i.e., $H_{1}$ holds),
the relative amount by which $q$ is closer to $r$ is so small that
a significant difference cannot be detected when $p$ and $q$ are
represented by samples of size $n=2000$, hence the low powers of
$\relume$ and $\relmmd$. Structural information provided by the
density functions allows $\relfssd$ (both $J=1$ and $J=5$) to detect
the difference even at $\epsilon=0.35$, as can be seen from the high
test powers. The fact that $\relmmd$ has higher power than $\relume$,
and the fact that changing $J$ from 1 to 5 increases the power only
slightly suggest that the differences may be spatially diffuse (rather
than local). 

%--------------------------
\textbf{3. \exsehist}  In this part, we demonstrate that test locations having positive
(negative) values of the power criterion correctly indicate the regions
in which $Q$ has a better (worse) fit. We consider image samples
from three categories of the CIFAR-10 dataset \citep{KriHin2009}:
airplane, automobile, and cat. We partition the images, and assume
that the sample from $P$ consists of 2000 airplane, 1500 cat images,
the sample from $Q$ consists of 2000 automobile, 1500 cat images,
and the reference sample from $R$ consists of 2000 automobile, 1500
cat images. All samples are independent. We consider a held-out random
sample consisting of 1000 images from each category, serving as a
pool of test location candidates. We set the kernel to be the Gaussian
kernel on 2048 features extracted by the Inception-v3 network at the
pool3 layer \citep{SzeVanIofShlWoj2016}. We evaluate the power criterion
of $\relume$ at each of the test locations in the pool individually.
The histogram of the criterion values is shown in Figure \ref{fig:cifar10_hist_ume_powcri}.
We observe that all the power criterion values are non-negative, confirming
that $Q$ is better than $P$ everywhere. Figure \ref{fig:cifar10_ume_powcri_0ascending}
shows the top 15 test locations as sorted in ascending order by the
criterion, consisting of automobile images. These indicate the regions
in the data domain where $Q$ fits better. Notice that cat images
do not have high positive criterion values because they can be modeled
equally well by $P$ and $Q$, and thus have scores close to zero
as shown in Figure \ref{fig:cifar10_ume_powcri_0ascending}.

%--------------------------
\textbf{4. \exseganrej{}} In this experiment, we apply the proposed $\relume$ test to comparing
two generative adversarial networks (GANs) \citep{GooPouMirXuWar2014}.
 We consider the CelebA dataset \citep{LiuLuoWanTan2015}\footnote{CelebA dataset: \url{http://mmlab.ie.cuhk.edu.hk/projects/CelebA.html}.}
in which each data point is an image of a celebrity with 40 binary
attributes annotated e.g., pointy nose, smiling, mustache, etc. We
create a partition of the images on the \emph{smiling} attribute,
thereby creating two disjoint subsets of\emph{ smiling} and \emph{non-smiling}
images. A set of 30000 images from each subset is held out for subsequent
relative goodness-of-fit testing, and the rest are used for training
two GAN models: a model for smiling images, and a model for non-smiling
images. Generated samples and details of the trained models can be
found in Section \ref{sec:models_s_ns} (appendix). The two models
are trained once and fixed throughout.

{\small{}}
\begin{table}
\centering{}{\small{}\caption{Rejection rates of the proposed $\protect\relume$, $\protect\relmmd$,
KID and FID, in the GAN model comparison problem. ``FID diff.''
refers to the average of $\mathrm{FID}(P,R)-\mathrm{FID}(Q,R)$ estimated
in each trial. Significance level $\alpha=0.01$ (for $\protect\relume$,
$\protect\relmmd$, and KID). \label{tab:gan_test_powers}}
\hspace{-2mm}}{\small{}}%
\begin{tabular}{ccc>{\centering}p{0.3cm}ccc>{\centering}p{0.5cm}>{\centering}p{0.5cm}>{\centering}p{0.5cm}c}
\toprule 
 & {\small{}$P$} & {\small{}$Q$} & {\small{}$R$} & \multicolumn{3}{c}{{\small{}$\relume$}} & \multirow{2}{0.5cm}{{\small{}Rel-MMD}} & {\small{}KID} & {\small{}FID } & {\small{}FID diff.}\tabularnewline
 &  &  &  & {\small{}J10} & {\small{}J20} & {\small{}J40} &  &  &  & \tabularnewline
\midrule
1. & {\small{}S} & {\small{}S} & {\small{}RS} & {\small{}0.0} & {\small{}0.0} & {\small{}0.0} & {\small{}0.0} & {\small{}0.0} & {\small{}0.53} & {\small{}-0.045 $\pm$ 0.52}\tabularnewline
2. & {\small{}RS} & {\small{}RS} & {\small{}RS} & {\small{}0.0} & {\small{}0.0} & {\small{}0.0} & {\small{}0.03} & {\small{}0.02} & {\small{}0.7} & {\small{}0.04 $\pm$ 0.19}\tabularnewline
3. & {\small{}S} & {\small{}N} & {\small{}RS} & {\small{}0.0} & {\small{}0.0} & {\small{}0.0} & {\small{}0.0} & {\small{}0.0} & {\small{}0.0} & {\small{}-15.22 $\pm$ 0.83}\tabularnewline
4. & {\small{}S} & {\small{}N} & {\small{}RN} & {\small{}0.57} & {\small{}0.97} & {\small{}1.0} & {\small{}1.0} & {\small{}1.0} & {\small{}1.0} & {\small{}5.25 $\pm$ 0.75}\tabularnewline
5. & {\small{}S} & {\small{}N} & {\small{}RM} & {\small{}0.0} & {\small{}0.0} & {\small{}0.0} & {\small{}0.0} & {\small{}0.0} & {\small{}0.0} & {\small{}-4.55 $\pm$ 0.82}\tabularnewline
\bottomrule
\end{tabular}{\small{}\vspace{-4mm}}
\end{table}
{\small \par}

In addition to $\relmmd$, we compare the proposed $\relume$ to Kernel
Inception Distance (KID) \citep{BinSutArbGre2018}, and Fr\'{e}chet
Inception Distance (FID) \citep{HeuRamUntNesHoc2017}, which are distances
between two samples (originally proposed for comparing a sample of
generated images, and a reference sample). All images are represented
by 2048 features extracted from the Inception-v3 network \citep{SzeVanIofShlWoj2016}
at the pool3 layer following \citet{BinSutArbGre2018}. When adapted
for three samples, KID is in fact a variant of $\relmmd$ in which
a third-order polynomial kernel is used instead of a Gaussian kernel
(on top of the pool3 features). Following \citet{BinSutArbGre2018},
we construct a bootstrap estimator for FID (10 subsamples with 1000
points in each). For the proposed $\relume$, the $J\in\{10,20,40\}$
test locations are randomly set to contain $J/2$ smiling images,
and $J/2$ non-smiling images drawn from a held-out set of real images.
We create problem variations by setting $P,Q,R\in\{$S, N, RS, RN,
RM$\}$ where S denotes generated smiling images (from the trained
model), N denotes generated non-smiling images, M denotes an equal
mixture of smiling and non-smiling images, and the prefix R indicates
that real images are used (as opposed to generated ones). The sample
size is $n=2000$, and each problem variation is repeated for 10 trials
for FID (due to its high complexity) and 100 trials for other methods.
The rejection rates from all the methods are shown in Table \ref{tab:gan_test_powers}.
Here, the test result for FID in each trial is considered ``reject
$H_{0}$'' if $\mathrm{FID}(P,R)>\mathrm{FID}(Q,R)$. \citet{HeuRamUntNesHoc2017}
did not propose FID as a statistical test. That said, there is a generic
way of constructing a relative goodness-of-fit test based on repeated
permutation of samples of $P$ and $Q$ to simulate from the null
distribution. However, FID requires computing the square root of the
feature covariance matrix (2048 x 2048), and is computationally too
expensive for permutation testing.

Overall, we observe that the proposed test does at least equally well
as existing approaches, in identifying the better model in each case.
In problems 1 and 2, $P$ and $Q$ have the same goodness of fit,
by design. In these cases, all the tests correctly yield low rejection
rates, staying roughly at the design level ($\alpha=0.01$). Without
a properly chosen threshold, the (false) rejection rates of FID fluctuate
around the expected value of 0.5. This means that simply comparing
FIDs (or other distances) to the reference sample without a calibrated
threshold can lead to a wrong conclusion on the relative goodness
of fit. The FID is further complicated by the fact that its estimator
suffers from bias in ways that are hard to model and correct for (see
\citet[Section D.1]{BinSutArbGre2018}). Problem 4 is a case where
the model $Q$ is better. We notice that increasing the number of
test locations of $\relume$ helps detect the better fit of $Q$.
In problem 5, the reference sample is bimodal, and each model can
capture only one of the two modes (analogous to the synthetic problem
in Figure \ref{fig:mmd_wit_isomix}). All the tests correctly indicate
that no model is better than another.

%----------------------------------------------------------------------------------------------
\textbf{5. \exseganfeature{}} 
\begin{figure}
\hspace{-2mm}\subfloat[Sample from $P=$ LSGAN trained for 15 epochs.\label{fig:lsgan15_sample}]{\includegraphics[width=0.25\textwidth]{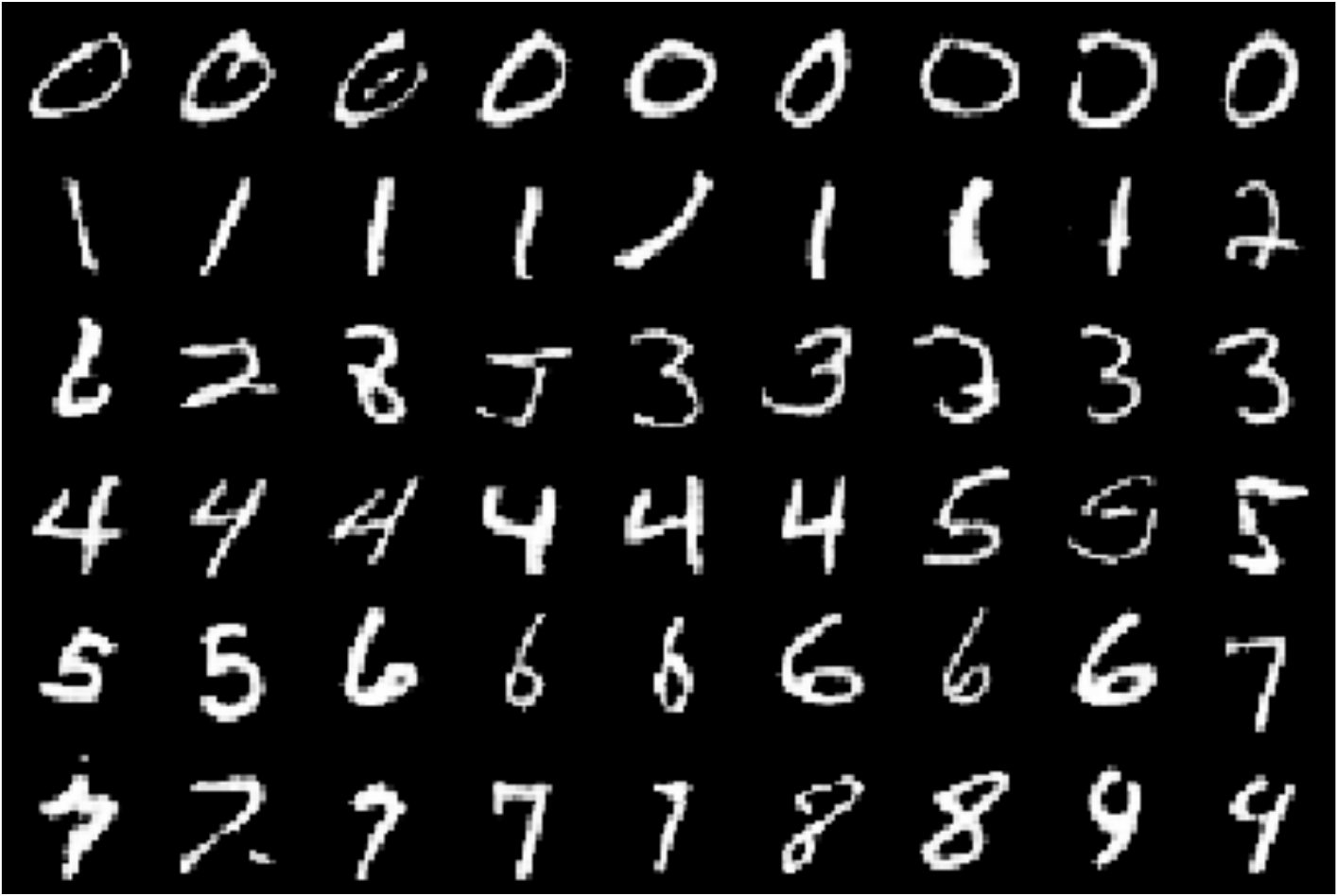}}\,\,\,\,\subfloat[Sample from $Q=$ LSGAN trained for 17 epochs.\label{fig:lsgan17_sample}]{\includegraphics[width=0.25\textwidth]{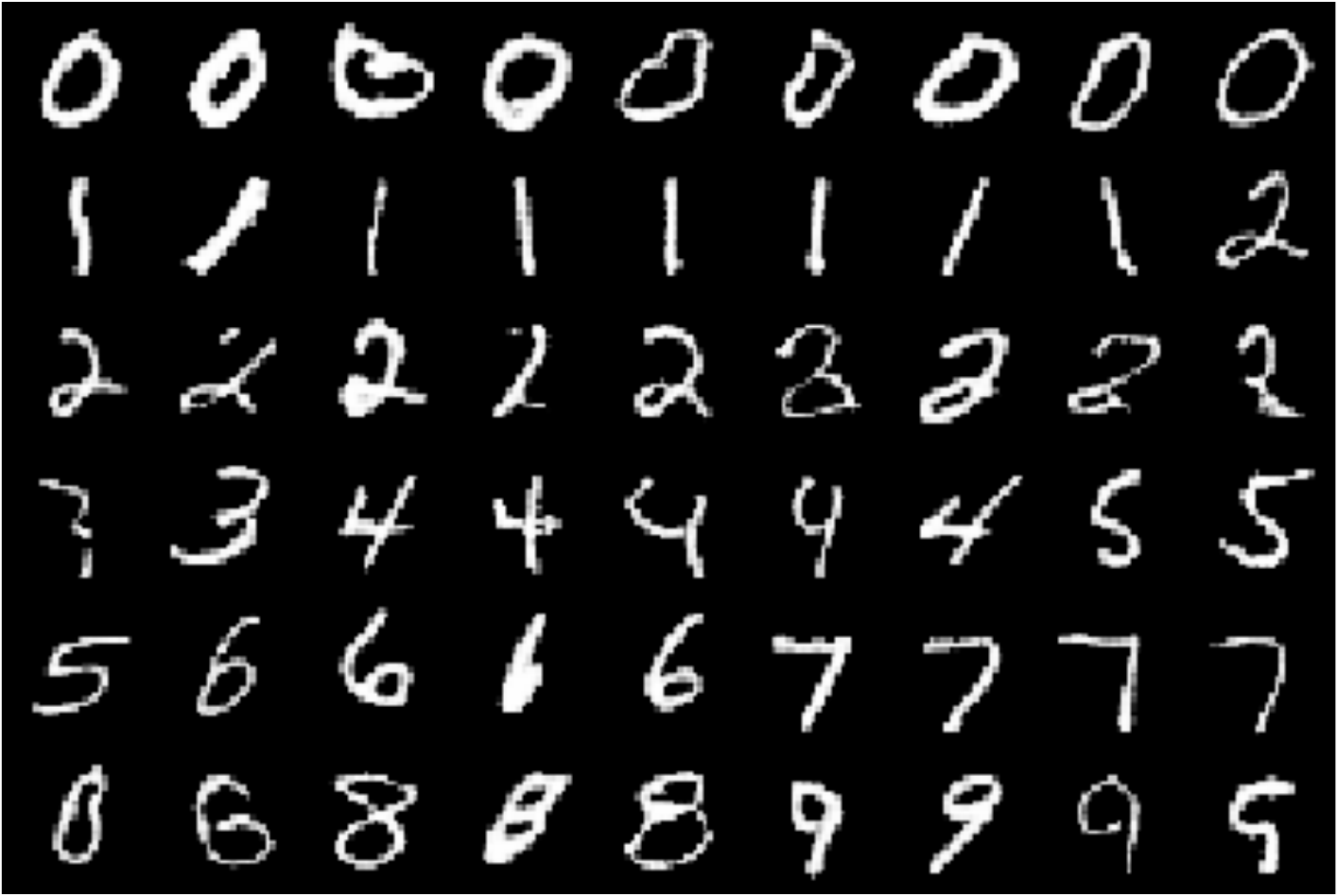}}\,\subfloat[Power criterion\label{fig:lsgan15_vs_17_box}]{\includegraphics[width=0.26\textwidth]{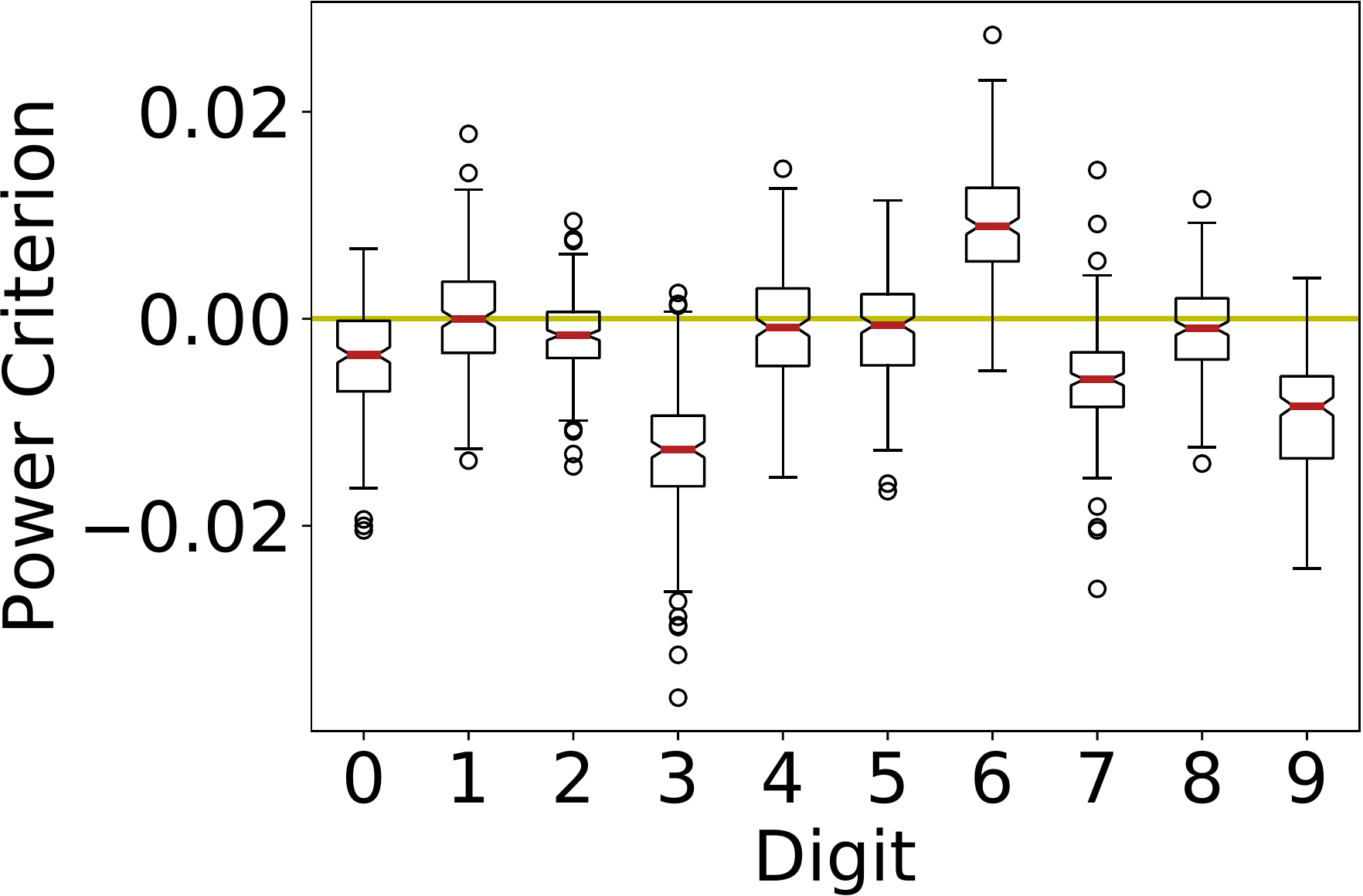}}\subfloat[Greedily optimized test locations\label{fig:lsgan15_vs_17_discopt}]{

\begin{tikzpicture}[scale=0.5, inner sep=0]
%\node[anchor=south west] at (16, 0) { };
\node[anchor=south east,inner sep=0] at (-0.1, 1) {Min:};
\node[anchor=south east,inner sep=0] at (-0.1, 3.3) {Max:};
\node[anchor=south west,inner sep=0] at (0,2.4) {
\includegraphics[width=0.17\textwidth,trim={0 180 0 5},clip]{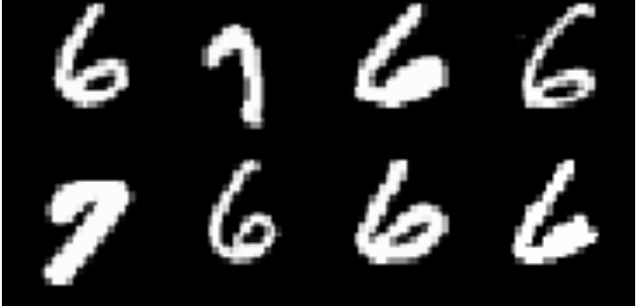}% 
};
\node[anchor=south west,inner sep=0] at (0,0) {
\includegraphics[width=0.17\textwidth,trim={0 180 0 5},clip]{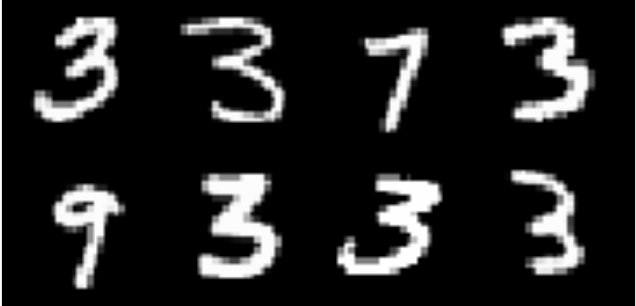}% 
};     
% grid lines to help
%\draw[help lines, gray, opacity=20] (0, 0) grid (18, 9); 
%\foreach \x in {0,1,...,18} { \node [anchor=north] at (\x,0) {\x}; } 
%\foreach \y in {0,1,...,9} { \node [anchor=east] at (0,\y) {\y}; }
\end{tikzpicture}

}\caption{Examining the training of an LSGAN model with $\protect\relume$.
(a), (b) Samples from the two models $P,Q$ trained on MNIST. (c)
Distributions of power criterion values computed over 200 trials.
Each distribution is formed by randomly selecting $J=40$ test locations
from real images of a digit type. (d) Test locations showing where
$Q$ is better (maximization of the power criterion), and test locations
showing where $P$ is better (minimization). \vspace{-7mm}}
\end{figure}
In the final experiment, we show that the power criterion of $\relume$
can be used to examine the relative change of the distribution of
a GAN model after training further for a few epochs. To illustrate,
we consider training an LSGAN model \citep{MaoLiXieLauWan2017} on
MNIST, a dataset in which each data point is an image of a handwritten
digit. We set $P$ and $Q$ to be LSGAN models after 15 epochs and
17 epochs of training, respectively. Details regarding the network
architecture, training, and the kernel (chosen to be a Gaussian kernel
on features extracted from a convolutional network) can be found in
Section \ref{sec:lsgan_mnist_details}. Samples from $P$ and $Q$
are shown in Figures \ref{fig:lsgan15_sample} and \ref{fig:lsgan17_sample}
(see Figure \ref{fig:lsgan_full_samples} in the appendix for more
samples). 

We set the test locations $V$ to be the set $V_{i}$ containing $J=40$
randomly selected real images of digit $i$, for $i\in\{0,\ldots,9\}$.
We then draw $n=2000$ points from $P,Q$ and the real data $(R)$,
and use $V=V_{i}$ to compute the power criterion for $i\in\{0,\ldots,9\}$.
The procedure is repeated for 200 trials where $V$ and the samples
are redrawn each time. The results are shown in Figure \ref{fig:lsgan15_vs_17_box}.
We observe that when $V=V_{3}$ (i.e., box plot at the digit 3) or
$V_{9}$, the power criterion values are mostly negative, indicating
that $P$ is better than $Q$, as measured in the regions indicated
by real images of the digits 3 or 9. By contrast, when $V=V_{6}$,
the large mass of the box plot in the positive orthant shows that
$Q$ is better in the regions of the digit 6. For other digits, the
criterion values spread around zero, showing that there is no difference
between $P$ and $Q$, on average. We further confirm that the class
proportions of the generated digits from both models are roughly correct
(i.e., uniform distribution), meaning that the difference between
$P$ and $Q$ in these cases is not due to the mismatch in class proportions
(see Section \ref{sec:lsgan_mnist_details}). These observations imply
that after the 15th epoch, training this particular LSGAN model two
epochs further improves generation of the digit 6, and degrades generation
of digits 3 and 9. A non-monotonic improvement during training is
not uncommon since at the 15th epoch the training has not converged.
More experimental results from comparing different GAN variants on
MNIST can be found in Section \ref{sec:mnist_gan_compare} in the
appendix.

We note that the set $V$ does not need to contain test locations
of the same digit. In fact, the notion of class labels may not even
exist in general. It is up to the user to define $V$ to contain examples
which capture the relevant concept of interest. For instance, to compare
the ability of models to generate straight strokes, one might include
digits 1 and 7 in the set $V$. An alternative to manual specification
of $V$ is to optimize the power criterion to find the locations that
best distinguish the two models (as done in experiment 2). To illustrate,
we consider greedily optimizing the power criterion by iteratively
selecting a test location (from real images) which best improves the
objective. Maximizing the objective yields locations that indicate
the better fit of $Q$, whereas minimization gives locations which
show the better fit of $P$ (recall from Figure \ref{fig:witness_illustration}).
The optimized locations are shown in Figure \ref{fig:lsgan15_vs_17_discopt}.
The results largely agree with our previous observations, and do not
require manually specifying $V$. This optimization procedure is applicable
to any models which can be sampled. 

%-------------------- end of main content -----------------

\subsubsection*{Acknowledgments}

HK and AG thank the Gatsby Charitable Foundation for the financial
support.

%----- References ---------

\bibliographystyle{abbrvnat}
\bibliography{kmod_nips2018}

%----- end of references -------

% ------------ appendix ---------------
\clearpage
\newpage
\appendix
%\newgeometry{left=3cm, right=3cm, top=2.5cm, bottom=2.5cm}t

\begin{center}
{\LARGE{}\ourtitle{}}
\par\end{center}{\LARGE \par}

\begin{center}
\textcolor{black}{\Large{}Supplementary}
\par\end{center}{\Large \par}

\section{Optimization of Test Locations in $\protect\relume$ and $\protect\relfssd$}

\label{sec:test_locs_opt} This section describes the optimization
procedure we use to select the test locations $V$ and the bandwidth
of the Gaussian kernel in the experiment ``\exsepowtoy{}.'' Since
the two sets $V,W$ of test locations are constrained to be the same
i.e., $V=W$ consisting of $J=J_{p}=J_{q}$ locations, in total, we
have $Jd+1$ parameters. We follow a similar implementation of the
optimization procedure for finding the test locations in FSSD.\footnote{Code for FSSD released by the authors: \url{https://github.com/wittawatj/kernel-gof}.}
All the parameters are optimized jointly by gradient ascent. We initialize
the test locations by randomly picking $J$ points from the training
set. The Gaussian width is initialized (for gradient ascent) to the
square of the mean of $\mathrm{med}_{X^{tr}\cup Z^{tr}}$ and $\mathrm{med}_{Y^{tr}\cup Z^{tr}}$,
where $\mathrm{med}_{A}:=\mathrm{median}\left(\left\{ \|\boldsymbol{x}-\boldsymbol{x}'\|_{2}\right\} _{\boldsymbol{x},\boldsymbol{x}'\in A}\right)$.
This is a similar heuristic used in \citet{BouBelBlaAntGre2015} to
set the bandwidth of the Gaussian kernel for $\relmmd$. 

\section{Trained Models for Generating Smiling and Non-Smiling Images}

\label{sec:models_s_ns}
\begin{figure}[H]
\begin{centering}
\subfloat[Samples from the smiling model]{\includegraphics[width=0.45\textwidth]{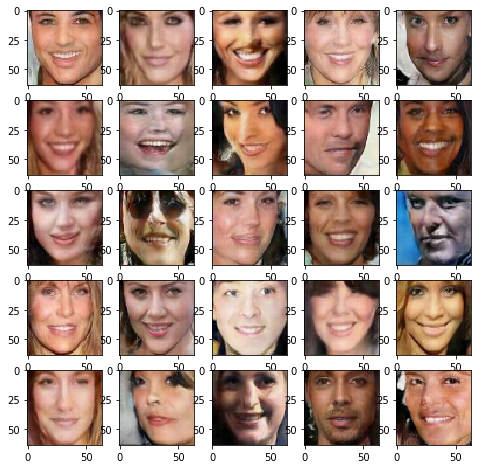}}\hspace{8mm}\subfloat[Samples from the non-smiling model]{\includegraphics[width=0.45\textwidth]{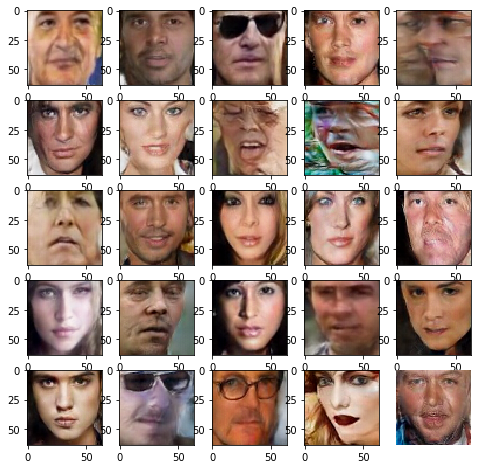}}
\par\end{centering}
\caption{Samples from the two trained models (smiling, and non-smiling) used
in "\exseganrej{}" experiment in Section \ref{sec:experiments}.
\label{fig:samples_s_ns_models}}
\end{figure}
This section describes the details of the two GAN models (smiling,
and non-smiling models) we use in the "\exseganrej{}" experiment
in Section \ref{sec:experiments}. We use the CelebA dataset \citep{LiuLuoWanTan2015}
in which each data point is an image of a celebrity with 40 binary
attributes annotated e.g., pointy nose, smiling, mustache, etc. We
create a partition of the images on the \emph{smiling} attribute,
thereby creating two disjoint subsets of\emph{ smiling} and \emph{non-smiling}
images. To reduce confounding factors that are not related to smiling
(e.g., sunglasses, background), each image is cropped to be 64x64
pixels, so that only the face remains. Cropping and image alignment
with eyes and lips are done with the software described in \citet{AmoLudSat2016}.
We use DCGAN architecture \citep{RadMetChi2015} (for both generator
and discriminator) for both smiling and non-smiling models, coded
in Pytorch. Subsampling was performed so that the training sizes for
the two models are equal. Each model is trained on 84,822 images (i.e.,
84822 smiling faces, and 84822 non-smiling faces) for 50 epochs. The
training time was roughly three hours using an Nvidia Titan X graphics
card with Pascal architecture. We use Adam optimizer \citep{KinBa2014}
with $\beta_{1}=0.5$ and $\beta_{2}=0.999$. The learning rate is
set to $10^{-3}$ (for both discriminator and generator in the two
models). Some samples generated from the two trained models are shown
in Figure \ref{fig:samples_s_ns_models}. 

\section{Proofs}

This section contains proofs for the results given in the main text.

\subsection{Proof of Theorem \ref{thm:normal_sume}}

\label{sec:proof_normal_sume}Let all the notations be defined as
in Section \ref{sec:new_mctests}. Recall Theorem \ref{thm:normal_sume}:
\normalsume*
%---------------------
\begin{proof}
Consider a random vector $\boldsymbol{t}:=(\boldsymbol{x},\boldsymbol{y},\boldsymbol{z})\in\mathcal{X}^{3}$,
where $\boldsymbol{x},\boldsymbol{y},$ and $\boldsymbol{z}$ are
independently drawn from $P,Q,$ and $R$, respectively. Let $T$
be the distribution of $\boldsymbol{t}$, and $\{\boldsymbol{t}_{i}\}_{i=1}^{n}=\{(\boldsymbol{x}_{i},\boldsymbol{y}_{i},\boldsymbol{z}_{i})\}_{i=1}^{n}\stackrel{i.i.d.}{\sim}T$.
Define two functions 
\begin{align*}
\delta_{V}^{P}(\boldsymbol{t},\boldsymbol{t}') & :=(\psi_{V}(\boldsymbol{x})-\psi_{V}(\boldsymbol{z}))^{\top}(\psi_{V}(\boldsymbol{x}')-\psi_{V}(\boldsymbol{z}')),\\
\delta_{W}^{Q}(\boldsymbol{t},\boldsymbol{t}') & :=(\psi_{W}(\boldsymbol{y})-\psi_{W}(\boldsymbol{z}))^{\top}(\psi_{W}(\boldsymbol{y}')-\psi_{W}(\boldsymbol{z}')),
\end{align*}
 where $\boldsymbol{t}':=(\boldsymbol{x}',\boldsymbol{y}',\boldsymbol{z}')$.
It can be seen that $\delta_{V}^{P}(\boldsymbol{t},\boldsymbol{t}')=\delta_{V}^{P}(\boldsymbol{t}',\boldsymbol{t})$
and $\delta_{W}^{Q}(\boldsymbol{t},\boldsymbol{t}')=\delta_{W}^{Q}(\boldsymbol{t}',\boldsymbol{t})$
for all $\boldsymbol{t},\boldsymbol{t}'\in\mathcal{X}^{3}$, and that
both functions are valid U-statistic kernels. It is not difficult
to see that $\umehpr$ and $\umehqr$ (estimator given in Section
\ref{sec:ume}) can be written in the form of second-order U-statistics
\citep[Chapter 5]{Ser2009} as
\begin{align*}
\umehpr & =\binom{n}{2}^{-1}\sum_{i=1}^{n}\sum_{j<i}\delta_{V}^{P}(\boldsymbol{t},\boldsymbol{t}'),\\
\umehqr & =\binom{n}{2}^{-1}\sum_{i=1}^{n}\sum_{j<i}\delta_{W}^{Q}(\boldsymbol{t},\boldsymbol{t}').
\end{align*}
Since $\psipv\neq\psirv$ (because $\umepr>0$), $\umehpr$ is a non-degenerate
U-statistic. Since $\psiqw\neq\psirw$, $\umehqr$ is also non-degenerate
\citep[Section 5.5.1]{Ser2009}. By \citet[Theorem 7.1]{Hoe1948},
asymptotically their joint distribution is given by a normal distribution:
\begin{equation}
\sqrt{n}\left(\left(\begin{array}{c}
\umehpr\\
\umehqr
\end{array}\right)-\left(\begin{array}{c}
\umepr\\
\umeqr
\end{array}\right)\right)\stackrel{d}{\to}\mathcal{N}\left(\boldsymbol{0},4\left(\begin{array}{cc}
\zeta_{P}^{2} & \zeta_{PQ}\\
\zeta_{PQ} & \zeta_{Q}^{2}
\end{array}\right)\right),\label{eq:joint_ume}
\end{equation}
where
\begin{align*}
\zeta_{P}^{2} & =\mathbb{V}_{\boldsymbol{t}\sim T}\left[\mathbb{E}_{\boldsymbol{t}'\sim T}[\delta_{V}^{P}(\boldsymbol{t},\boldsymbol{t}')]\right]\stackrel{(a)}{=}(\psipv-\psirv)^{\top}(\covpv+\covrv)(\psipv-\psirv),\\
\zeta_{Q}^{2} & =\mathbb{V}_{\boldsymbol{t}\sim T}\left[\mathbb{E}_{\boldsymbol{t}'\sim T}[\delta_{W}^{Q}(\boldsymbol{t},\boldsymbol{t}')]\right]\stackrel{(b)}{=}(\psiqw-\psirw)^{\top}(\covqw+\covrw)(\psiqw-\psirw),\\
\zeta_{PQ} & =\mathrm{cov}_{\boldsymbol{t}\sim T}\left(\mathbb{E}_{\boldsymbol{t}'\sim T}[\delta_{V}^{P}(\boldsymbol{t},\boldsymbol{t}')],\mathbb{E}_{\boldsymbol{t}'\sim T}[\delta_{W}^{Q}(\boldsymbol{t},\boldsymbol{t}')]\right)\stackrel{(c)}{=}(\psipv-\psirv)^{\top}C_{VW}^{R}(\psiqw-\psirw),
\end{align*}
and $C_{VW}^{R}:=\mathrm{cov}_{\boldsymbol{z}\sim R}[\psi_{V}(\boldsymbol{z}),\psi_{W}(\boldsymbol{z})]\in\mathbb{R}^{J_{p}\times J_{q}}.$
At $(a),(b),(c)$, we rely on the independence among $\boldsymbol{x},\boldsymbol{y},$
and $\boldsymbol{z}$. A direct calculation gives the expressions
of $\zeta_{P}^{2}$, $\zeta_{Q}^{2}$, and $\zeta_{PQ}$. By the continuous
mapping theorem, and (\ref{eq:joint_ume}), $\sqrt{n}\left(\begin{array}{c}
1\\
-1
\end{array}\right)^{\top}\left(\left(\begin{array}{c}
\umehpr\\
\umehqr
\end{array}\right)-\left(\begin{array}{c}
\umepr\\
\umeqr
\end{array}\right)\right)=\sqrt{n}\left(\widehat{S}_{n}^{U}-S^{U}\right)\stackrel{d}{\to}\mathcal{N}\left(\boldsymbol{0},4\left(\begin{array}{c}
1\\
-1
\end{array}\right)^{\top}\left(\begin{array}{cc}
\zeta_{P}^{2} & \zeta_{PQ}\\
\zeta_{PQ} & \zeta_{Q}^{2}
\end{array}\right)\left(\begin{array}{c}
1\\
-1
\end{array}\right)\right)$ giving the result.
\end{proof}
\emph{Remark 1}. The assumption that $P,Q,$ and $R$ are all distinct
in Theorem \ref{thm:normal_sume} is necessary for $\umehpr$ and
$\umehqr$ to follow a non-degenerate normal distribution asymptotically.
If $R\in\{P,Q\}$, then $\widehat{U_{S}^{2}}$ for $S\in\{P,Q\}$
asymptotically follows a weighted sum of chi-squared random variables,
and $U_{S}^{2}=0$. If $P=Q$, the covariance matrix in (\ref{eq:joint_ume})
is rank-defficient.

\section{Details of Experiment 5: \exseganfeature{}}

\label{sec:lsgan_mnist_details}

\paragraph{LSGAN Architecture}

We rely on Pytorch code\footnote{\url{https://github.com/znxlwm/pytorch-generative-model-collections}
(commit: 0d183bb5ea)} by Hyeonwoo Kang to train the LSGAN \citep{MaoLiXieLauWan2017} model
that we use in experiment 5. Network architectures of the generator
and the discriminator follow the design used in \citet[Section C.1]{CheDuaHouSchSut2016}.
We reproduce here in Table \ref{tab:lsgan_net_arch} for ease of reference.

\begin{table}
\begin{centering}
\caption{Discriminator and generator of LSGAN used in experiment 5.\label{tab:lsgan_net_arch}}
\begin{tabular}{ll}
\toprule 
Discriminator & Generator\tabularnewline
\midrule
\midrule 
Input: $28\times28$ grayscale image & Input noise vector $\boldsymbol{z}\sim\mathrm{Unif}[0,1]^{62}$ \tabularnewline
$4\times4$ conv. 64 LRELU. Stride 2. & FC. 1024 RELU. Batch norm.\tabularnewline
$4\times4$ conv. 128 LRELU. Stride 2. Batch norm. & FC. $7\times7\times128$ RELU. Batch norm.\tabularnewline
FC. 104 Leaky RELU. Batch norm. & $4\times4$ upconv. 64 RELU. Stride 2. Batch norm.\tabularnewline
FC & $4\times4$ upconv. 1 channel.\tabularnewline
\bottomrule
\end{tabular}
\par\end{centering}
\raggedright{}\vspace{2mm}conv. refers to a convolution layer, FC
means a fully-connected layer, RELU means a rectified linear unit,
LRELU means Leaky RELU, and upconv is the transposed convolution.
\end{table}

\paragraph{Kernel Function }

The kernel $k$ is chosen to be a Gaussian kernel on features extracted
from a convolutional neural network (CNN) classifier trained to classify
the ten digits of MNIST. Specifically the kernel $k$ is $k(\boldsymbol{x},\boldsymbol{y})=\exp\left(-\frac{\|f(\boldsymbol{x})-f(\boldsymbol{y})\|_{2}^{2}}{2\nu^{2}}\right)$,
where $f$ is the output (in $\mathbb{R}^{10}$) of the last fully-connected
layer of a trained CNN classifier.\footnote{Code to train the CNN classifier is taken from \url{https://github.com/pytorch/examples/blob/master/mnist/main.py}
(commit: 75e7c75).} The architecture of the CNN is
\begin{align*}
\text{Input: 28\ensuremath{\times}28 grayscale image} & \to5\times5\text{ conv. 10 filters. 2\ensuremath{\times}2\text{ max pool}}\\
 & \to5\times5\text{ conv. 20 filters. }\text{2\ensuremath{\times}2 \text{max pool}}\\
 & \to\text{FC. 50 RELU.}\\
 & \to\text{FC. }10\text{ outputs}.
\end{align*}
We train the CNN for 30 epochs and achieve higher than 99\% accuracy
on MNIST's test set. The Gaussian bandwidth $\nu$ is set with the
median heuristic.

\paragraph{Class Proportion of Generated Digits}

To examine the proportion of digits in the generated samples, we sample
4000 images from both models $P$ (LSGAN-15, LSGAN model trained for
15 epochs), and $Q$ (LSGAN-17, LSGAN model trained for 17 epochs),
and use the CNN classifier to assign a label to each image. The proportions
of digits are shown in Figure \ref{fig:lsgan_15vs17_class_prop}.
We observe that the generated digits from both LSGAN-15 and LSGAN-17
follow the right distribution i.e., uniform distribution, up to variability
due to noise. There is no mode collapse problem. This observation
means that the difference between $P$ and $Q$ studied in experiment
5 in the main text is not due to the mismatch of class proportions.

\begin{figure}
\centering{}\includegraphics[width=0.4\textwidth]{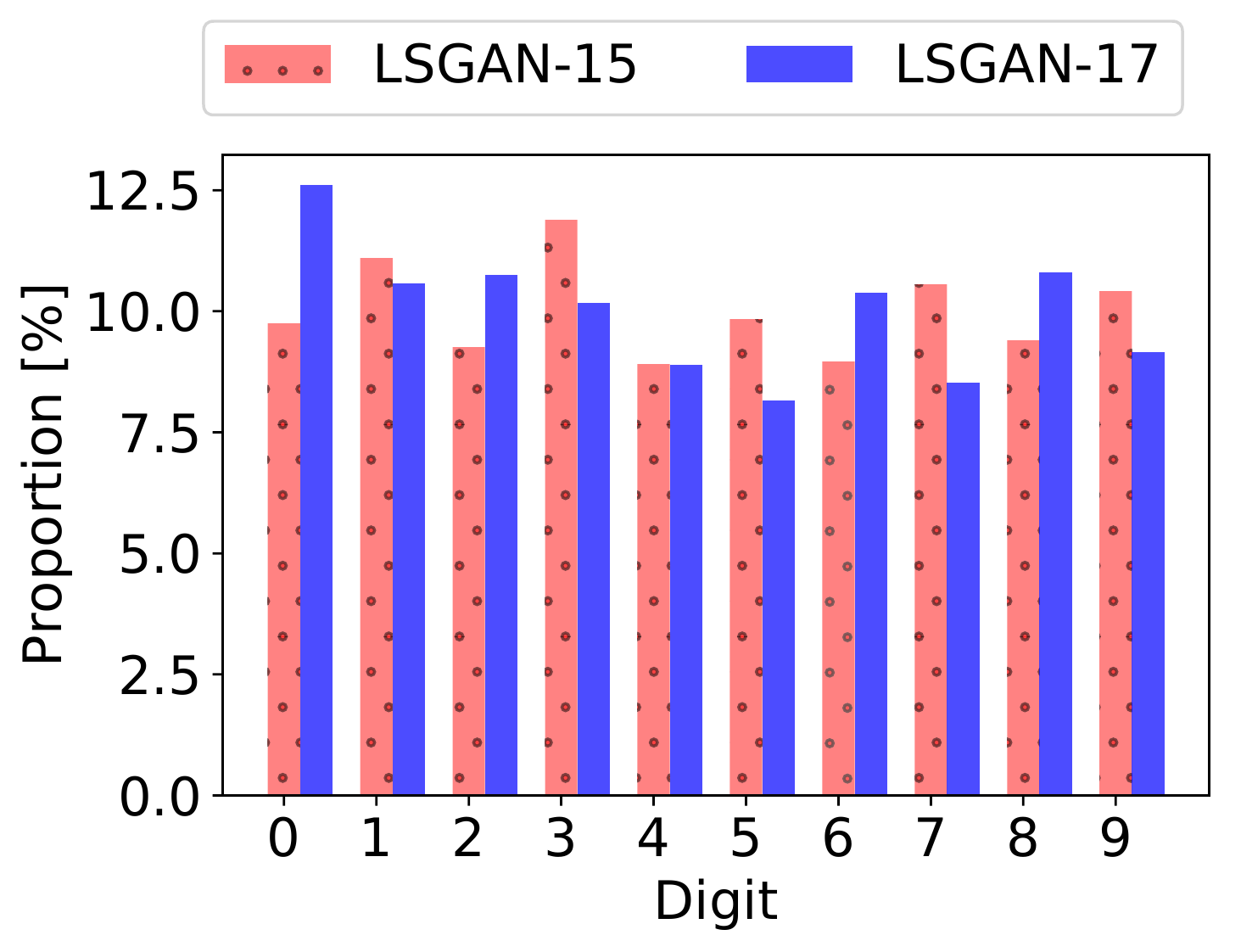}\caption{Proportions of generated digits from the LSGAN models at 15th and
17th epochs. Classification of each generated image is done by a trained
convolutional neural network classifier (see Section \ref{sec:lsgan_mnist_details}).\label{fig:lsgan_15vs17_class_prop}}
\end{figure}

\begin{figure}[H]
\begin{centering}
\subfloat[Samples from LSGAN trained for 15 epochs.]{\includegraphics[width=0.8\textwidth]{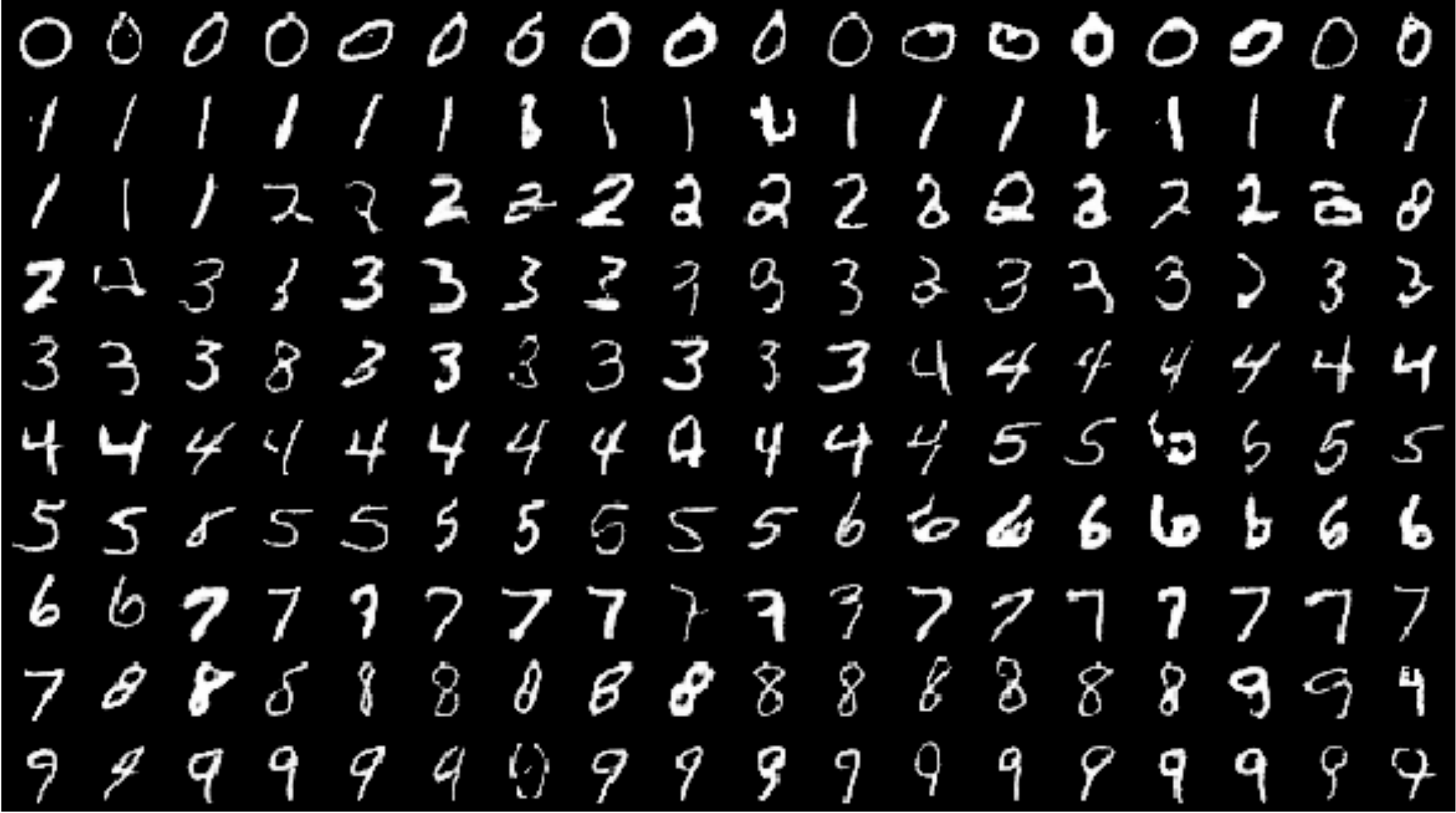}}\hspace{8mm}\subfloat[Samples from LSGAN trained for 17 epochs.]{\includegraphics[width=0.8\textwidth]{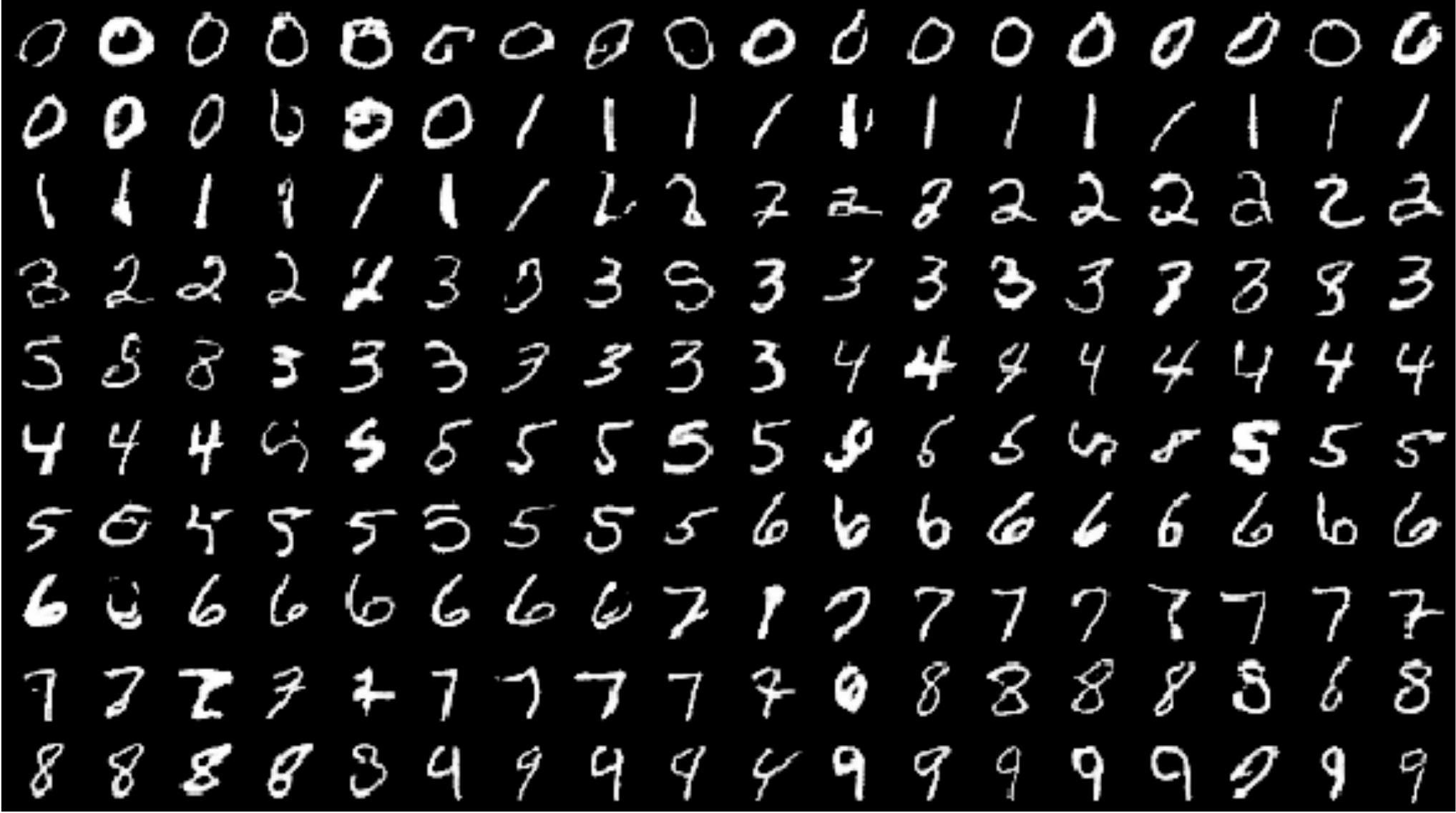}}
\par\end{centering}
\caption{Samples from LSGAN models trained on MNIST. Samples are taken from
the models at two different time points: after 15 epochs, and after
17 epochs of training. \label{fig:lsgan_full_samples}}
\end{figure}

\section{Comparing Different GAN Models Trained on MNIST}

\label{sec:mnist_gan_compare}This section extends experiment 5 in
the main text to compare other GAN variants trained on MNIST. All
the GAN variants that we consider have the same network architecture
as described in Table \ref{tab:lsgan_net_arch}. We use the notation
\emph{AAA-n} to refer to a GAN model of type AAA trained for $n$
epochs. We note that the result presented here for each GAN variant
does not represent its best achievable result. 

\paragraph{WGAN-GP-10 vs LSGAN-10}

Here we compare $P=$ Wasserstein GAN with Gradient Penalty \citep{GulAhmArjDumCou2017}
and $Q=$ LSGAN \citep{MaoLiXieLauWan2017} trained for ten epochs
on MNIST. The results are shown in Figure \ref{fig:wgangp10_vs_lsgan10}.
From the generated samples from the two models, it appears that LSGAN
yields more realistic images of handwritten digits, after training
for ten epochs. The positive power criterion values in Figure \ref{fig:wgangp10_vs_lsgan10_box}
further confirm this observation i.e., $Q$ is better at all digits. 

\begin{figure}[h]
\centering{}\vspace{-3mm}\subfloat[Sample from $P=$ WGAN-GP trained for 10 epochs.]{\includegraphics[width=0.32\textwidth]{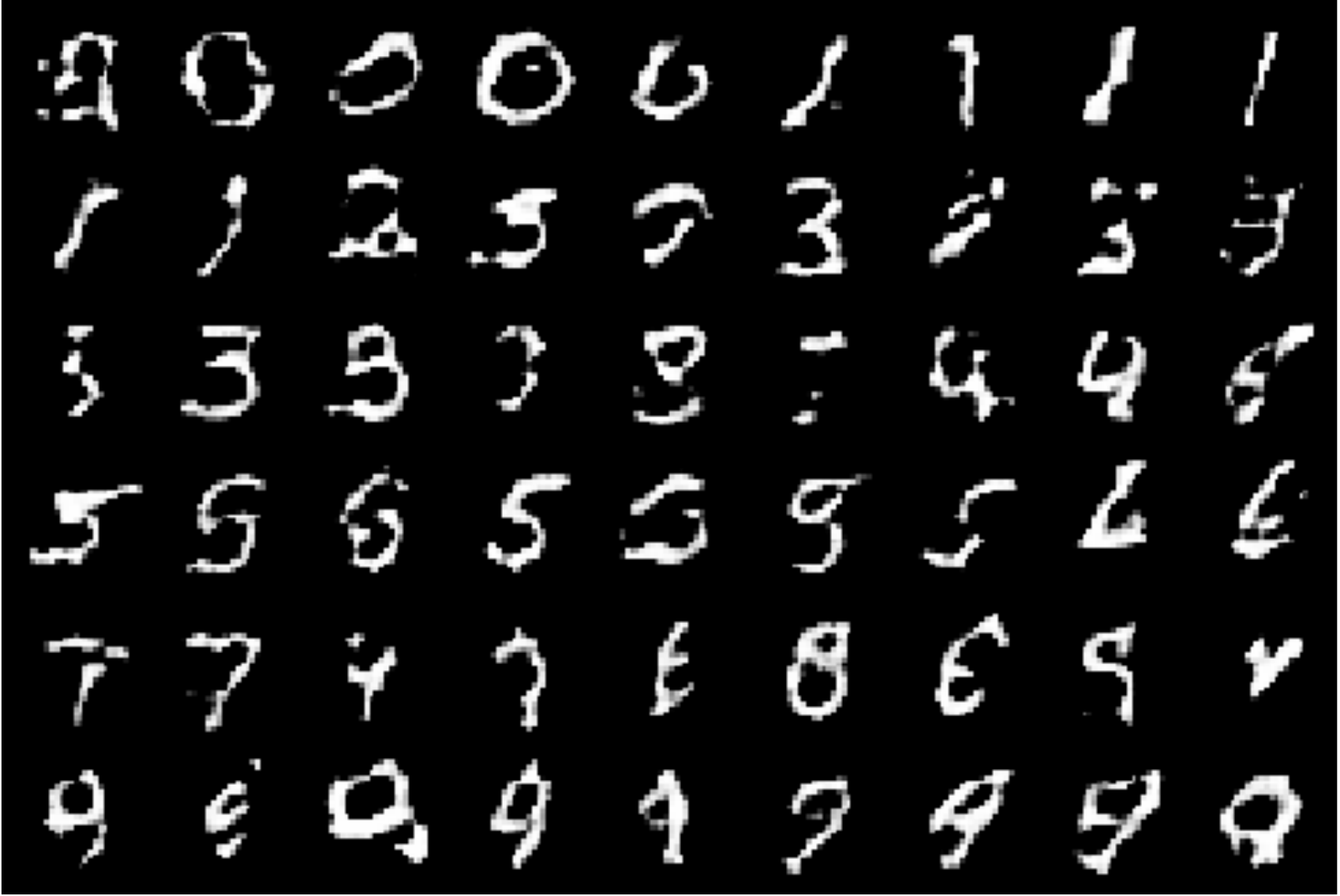}}\,\,\,\,\subfloat[Sample from $Q=$ LSGAN trained for 10 epochs.]{\includegraphics[width=0.32\textwidth]{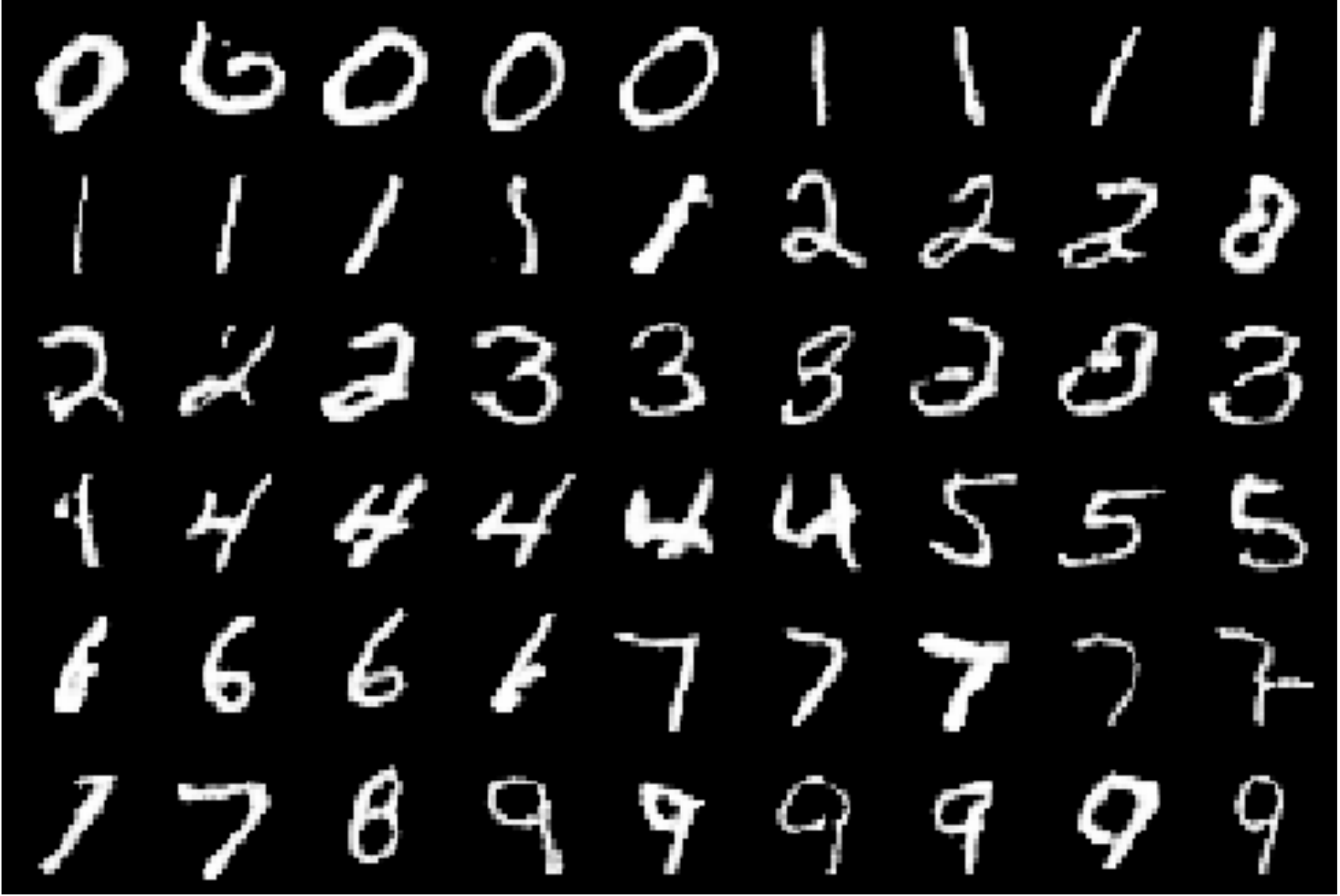}}\,\subfloat[Power criterion \label{fig:wgangp10_vs_lsgan10_box}]{\includegraphics[width=0.31\textwidth]{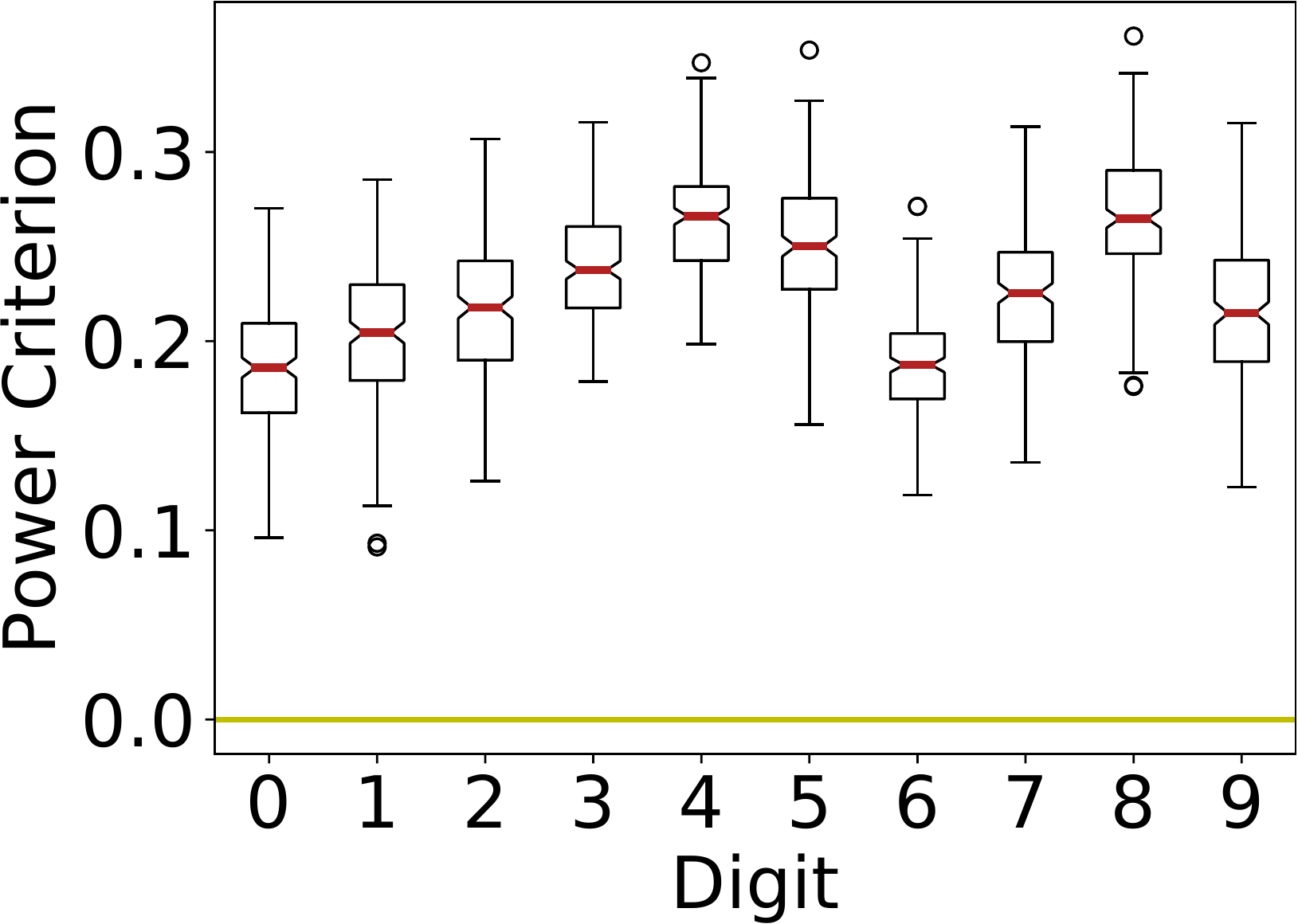}}\caption{Comparing WGAN-GP (Wasserstein GAN with Gradient Penalty) and LSGAN,
trained for ten epochs on MNIST. \label{fig:wgangp10_vs_lsgan10}\vspace{-3mm}}
\end{figure}

\paragraph{GAN-40 vs LSGAN-40}

In this part, we compare $P=$ GAN-40 \citep{GooPouMirXuWar2014}
and $Q=$ LSGAN trained for 40 epochs on MNIST. The results are shown
in Figure \ref{fig:gan40_vs_lsgan40}. It can be seen from visual
inspection that LSGAN-40 is slightly better overall, except for digits
1 and 5 at which LSGAN-40 appears to be significantly better. This
observation is also hinted by the power criterion values at digits
1 and 5 which tend to be positive (see Figure \ref{fig:gan40_vs_lsgan40_box}).

\begin{figure}[h]
\centering{}\vspace{-3mm}\subfloat[Sample from $P=$ GAN trained for 40 epochs.]{\includegraphics[width=0.32\textwidth]{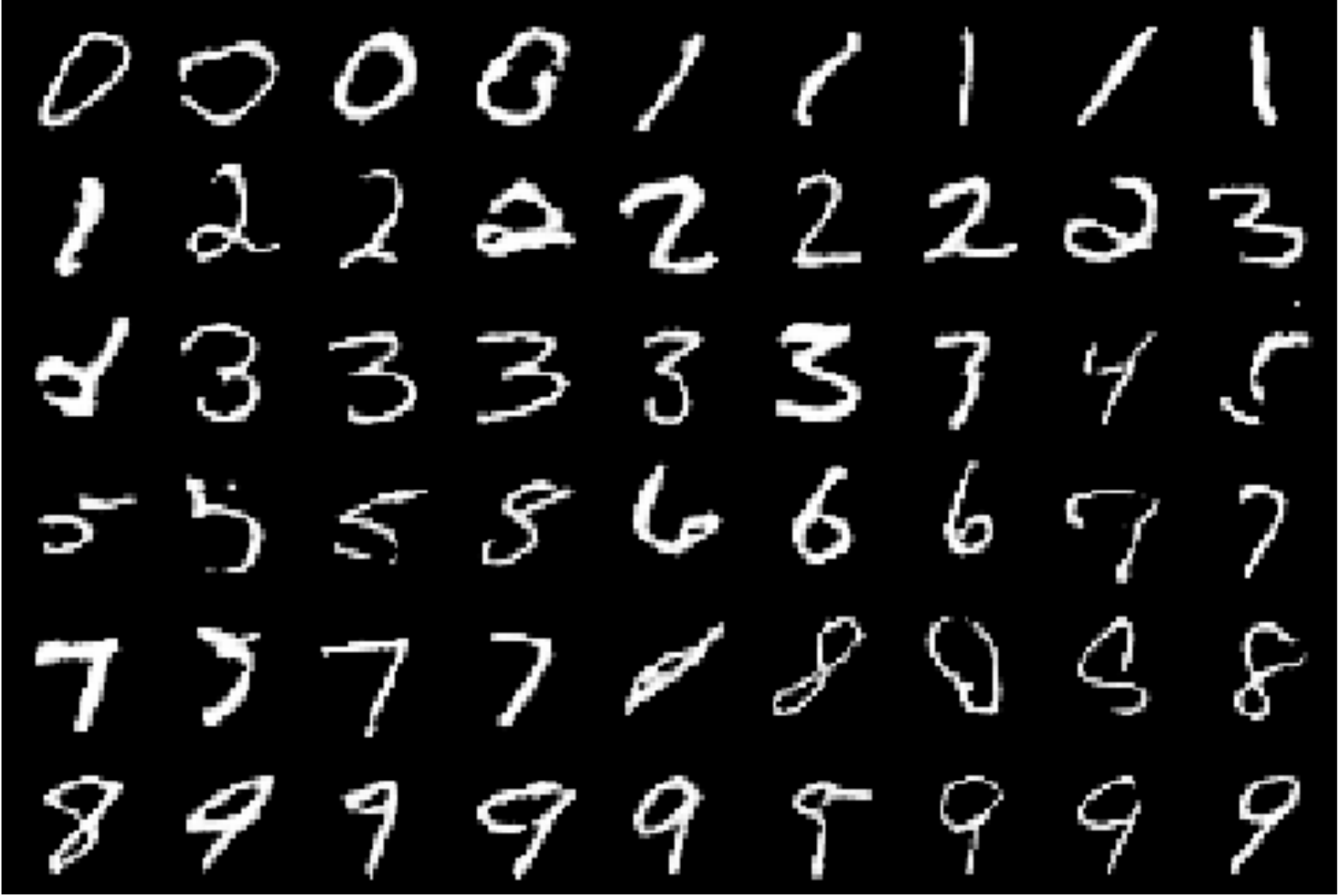}}\,\,\,\,\subfloat[Sample from $Q=$ LSGAN trained for 40 epochs.]{\includegraphics[width=0.32\textwidth]{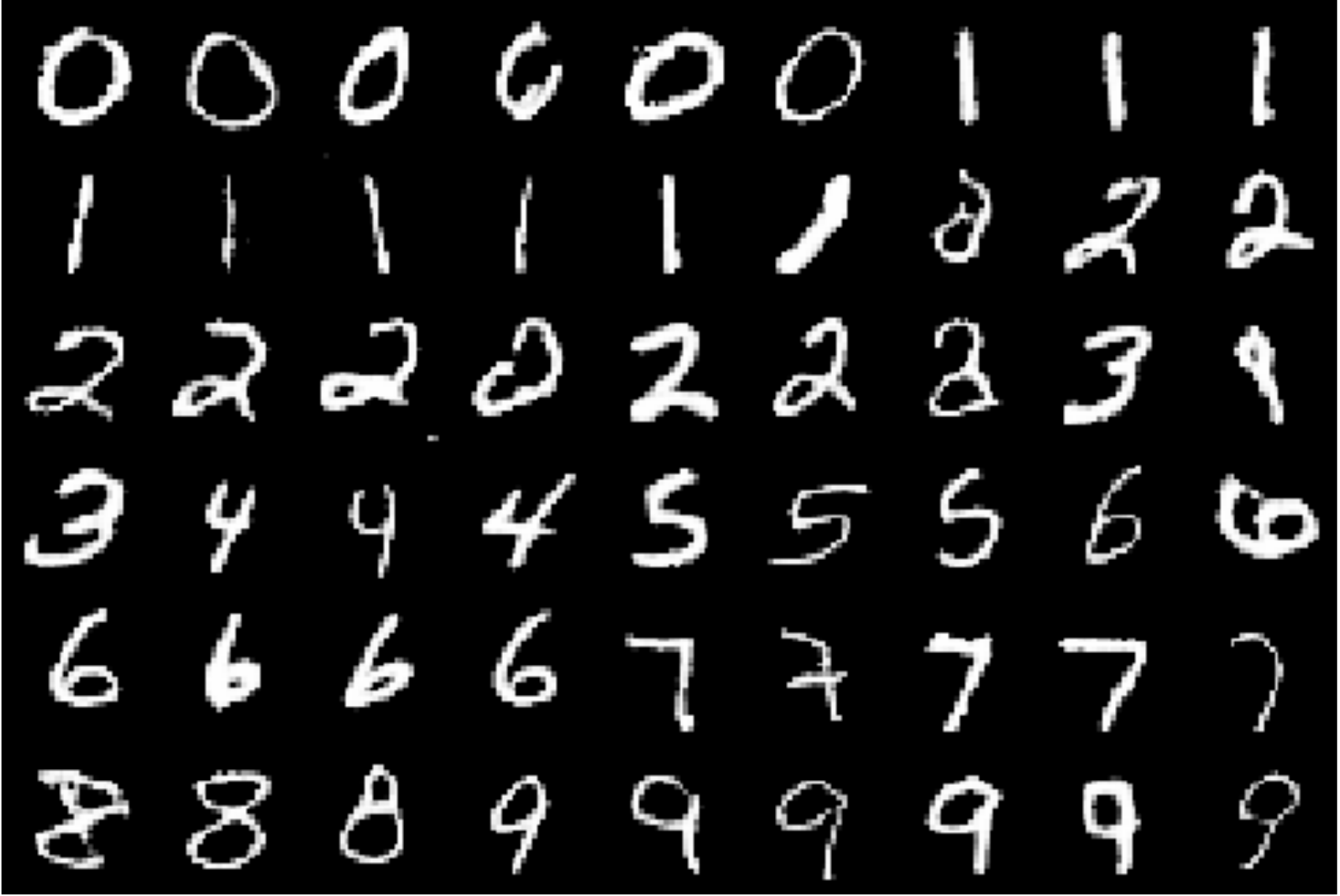}}\,\subfloat[Power criterion \label{fig:gan40_vs_lsgan40_box}]{\includegraphics[width=0.33\textwidth]{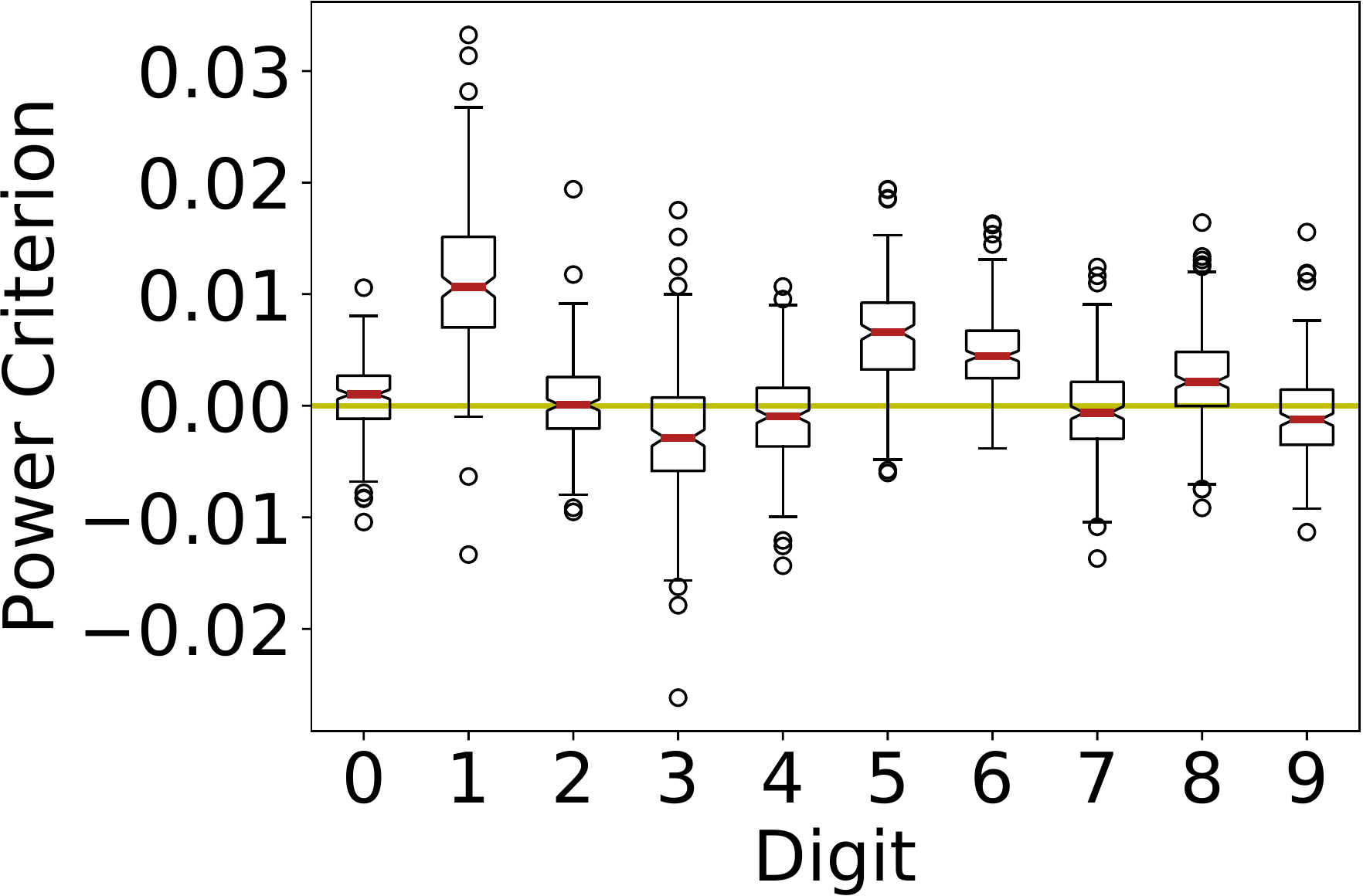}}\caption{Comparing GAN (the original formulation) and LSGAN, trained for 40
epochs on MNIST. \label{fig:gan40_vs_lsgan40}\vspace{-3mm}}
\end{figure}

\paragraph{WGAN-30 vs WGAN-30}

As a sanity check, we also run the same procedure on a case where
$P=Q$. We set $P=Q=$ Wasserstein GAN (WGAN, \citep{wgan} trained
for 30 epochs on MNIST. The results are shown in Figure \ref{fig:wgan30_vs_wgan30}.
As expected, the power criterion values spread around zero in all
cases. We note that we did not modify the procedure to treat this
special case. In particular, in each trial, two samples are drawn
from $P$ and $Q$ as usual.

\begin{figure}[h]
\centering{}\vspace{-3mm}\subfloat[Sample from $P=Q=$ WGAN trained for 30 epochs.]{\includegraphics[width=0.64\textwidth]{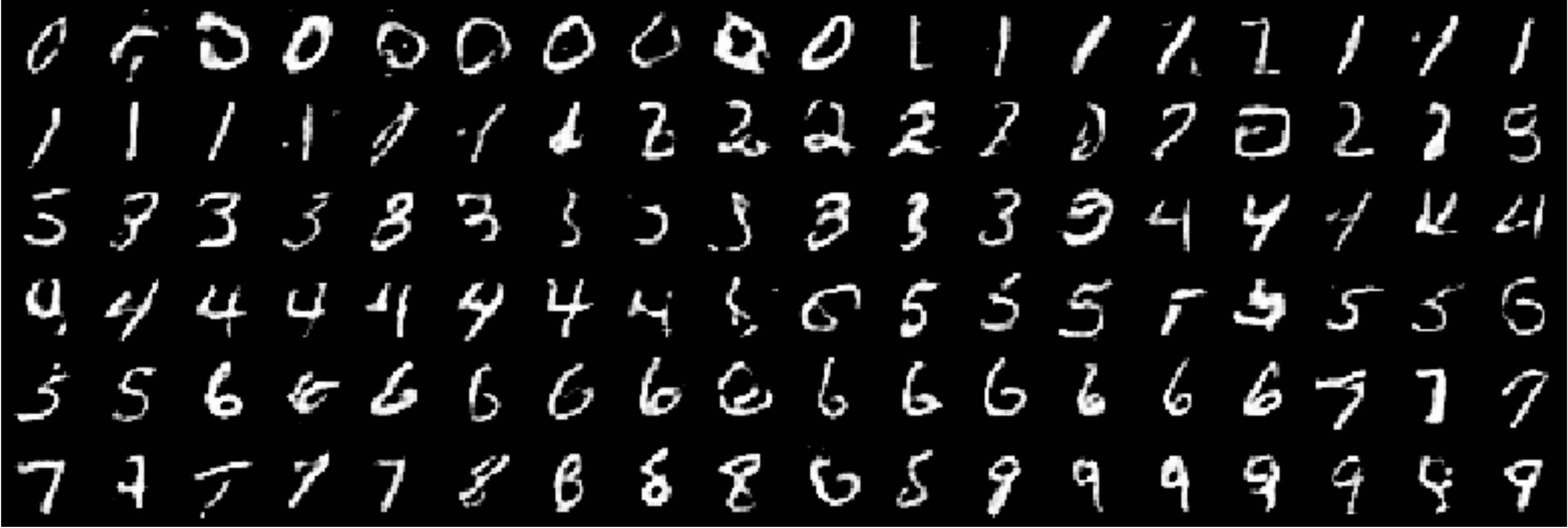}}\,\subfloat[Power criterion \label{fig:wgan30_vs_wgan30_box}]{\includegraphics[width=0.33\textwidth]{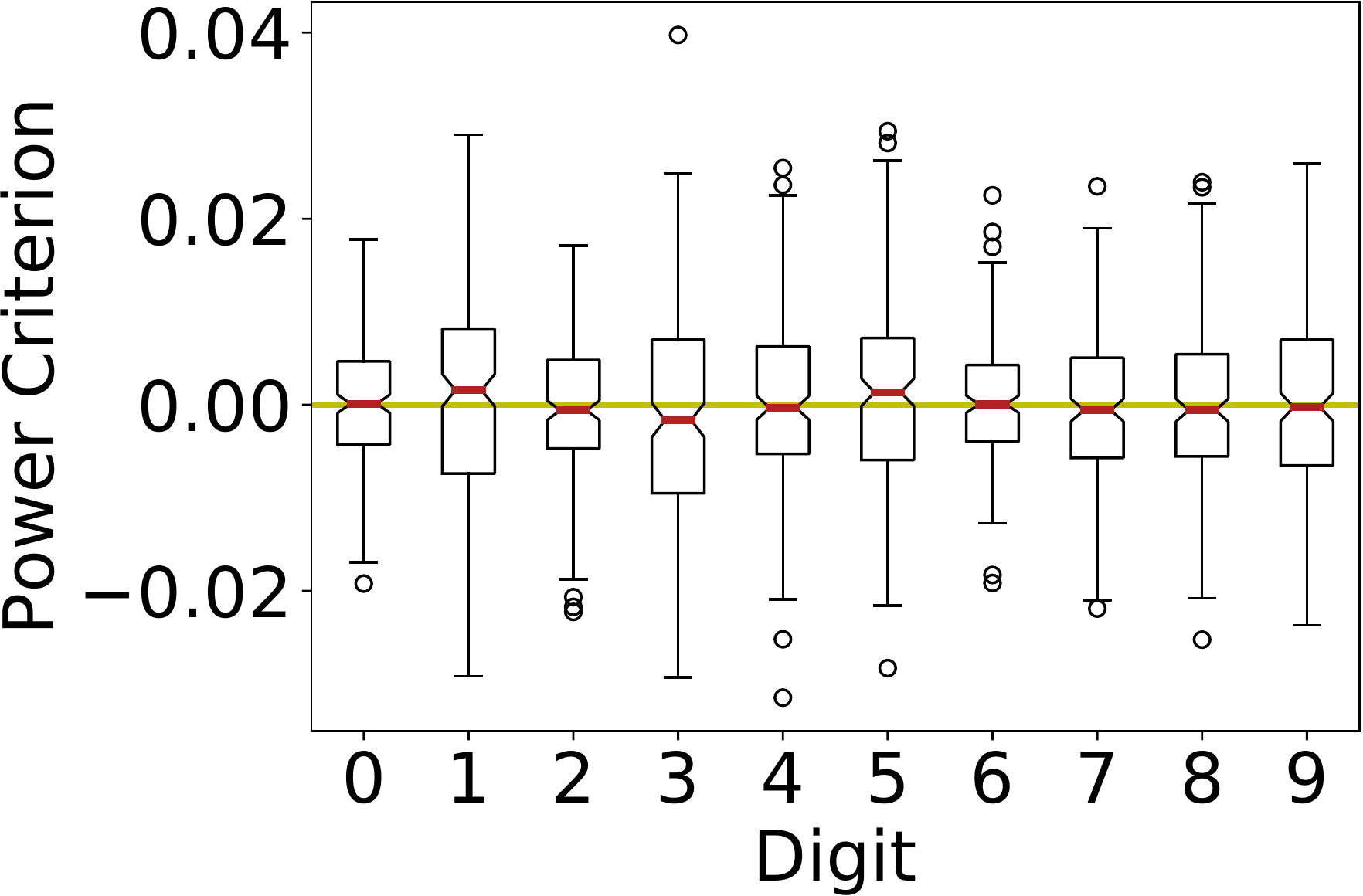}}\caption{Comparing two models which are the same for sanity checking. The model
is set to WGAN trained for 30 epochs. \label{fig:wgan30_vs_wgan30}\vspace{-3mm}}
\end{figure}

\end{document}